\documentclass[times,twocolumn,final,authoryear]{elsarticle}

\usepackage{geometry}
\usepackage{framed,multirow}

\usepackage{amsfonts,amsmath,amsthm,amssymb,bm}
\usepackage{latexsym}

\usepackage{color, xcolor, colortbl}
\definecolor{newcolor}{rgb}{.8,.349,.1}

\journal{Pattern Recognition Letters}

\usepackage{enumitem}
\usepackage{subcaption}

\usepackage{hyperref,url}
\hypersetup{
  colorlinks,
  citecolor=blue,
  filecolor=blue,
  linkcolor=blue,
  urlcolor=blue
}

\newtheorem{theorem}{Theorem}

\newtheorem{lemma}[theorem]{Lemma}

\newcommand{\APPENDIX}{Appendix}
\graphicspath{{./Graphics/}}

\begin{document}

\begin{frontmatter}
  \title{Estimating the standard error of cross-Validation-Based estimators of classifier performance}

  \author[WAY]{Waleed~A.~Yousef\corref{cor1}}
  \ead{wyousef@UVIC.ca, wyousef@fci.helwan.edu.eg}
  \cortext[cor1]{Corresponding Author}

  \address[WAY]{Ph.D., ECE Dep., University of Victoria, Canada;CS Dep.,Human Computer Interaction
    Laboratory (HCILAB), Faculty of Computers and Artificial Intelligence, Helwan University,
    Egypt.}

  \begin{abstract}
    First, we analyze the variance of the Cross Validation (CV)-based estimators used for estimating
    the performance of classification rules. Second, we propose a novel estimator to estimate this
    variance using the Influence Function (IF) approach that had been used previously very
    successfully to estimate the variance of the bootstrap-based estimators. The motivation for this
    research is that, as the best of our knowledge, the literature lacks a rigorous method for
    estimating the variance of the CV-based estimators. What is available is a set of ad-hoc
    procedures that have no mathematical foundation since they ignore the covariance structure among
    dependent random variables. The conducted experiments show that the IF proposed method has small
    RMS error with some bias. However, surprisingly, the ad-hoc methods still work better than the
    IF-based method. Unfortunately, this is due to the lack of enough smoothness if compared to the
    bootstrap estimator. This opens the research for three points: (1) more comprehensive simulation
    study to clarify when the IF method win or loose; (2) more mathematical analysis to figure out
    why the ad-hoc methods work well; and (3) more mathematical treatment to figure out the
    connection between the appropriate amount of ``smoothness'' and decreasing the bias of the IF
    method.
  \end{abstract}

\end{frontmatter}


\section{Introduction}\label{sec:introduction}
\subsection{Background and Motivation}\label{sec:backgr-motiv}
As was introduced in \cite{Yousef2019LeisurelyLookVersionsVariants-arxiv}, we tried to set the different
Cross Validation (CV)-based estimators in one mathematical picture to understand the difference
among them. The present article aims at providing a rigorous framework for estimating the standard
error associated with those estimators. We published some of the preliminary experimentation early
in \cite{Yousef2009EstCVvariability}, without providing any theoretical foundation. Therefore the
present article is, in particular, the theoretical foundation and the full experimental extension to
\cite{Yousef2009EstCVvariability}; and, in general, a continuation to our methods of assessing
uncertainty of classifiers performance and their estimators
\citep{Yousef2013PAUC,Yousef2009EstCVvariability,Yousef2009NonparEstThreshold,Yousef2006AssessClass,Yousef2005EstimatingThe,Yousef2004ComparisonOf}.

\bigskip

The motivation for this research is that, as the best of our knowledge, the literature lacks a
rigorous method for estimating the variance of the CV-based estimators. What is available is a set
of ad-hoc procedures that have no mathematical foundation---since they ignore the covariance
structure among dependent random variables, i.e., the folds of the CV. One of these methods, e.g.,
is using the simple sample variance method among the folds. \cite{Bengio2004NoUnbiasedEstKCV} is an
early work to analyze this covariance structure of the CV.

\cite{Efron1997ImprovementsOnCross} is the first to use the IF approach to estimate
the standard error of their Leave One Out Bootstrap (LOOB) estimator that estimates the error rate
of a classification rule. \cite{Yousef2005EstimatingThe} then extended the same approach to estimate
the standard error of their Leave Pair Out Bootstrap (LPOB) estimator that estimate the Area Under
the ROC Curve (AUC) of a classification rule. The \textit{smoothness} issue of booth estimators,
LOOB and LPOB, is at the kernel of success of the IF approach. \textit{Smoothness} is simply the
averaging implied on each observation from the design of the resampling mechanism of both
bootstrap-based estimators. \textit{Smoothness} is well explained in~\cite{Yousef2019AucSmoothness-arxiv}.

On the other hand, In \cite{Yousef2019LeisurelyLookVersionsVariants-arxiv}, we show that how two different
versions of the CV estimators inherits this \textit{smoothness} feature from the resampling
mechanism of the CV. These two versions are the Monte-Carlo $K$-fold Cross Validation (CVKM) and the
Repeated $K$-fold Cross Validation (CVKR)---both converge a.s. to the same value as we proved in
\cite{Yousef2019LeisurelyLookVersionsVariants-arxiv}. This motivated us to develop an IF-based method for
estimating the standard error of these two CV-based estimators, analogous to the bootstrap based
estimators. This way, we can abandon the ad-hoc methods mentioned above. Then, we conducted a
simulation study to compare the true standard error (obtained from Monte Carlo), the ad-hoc
estimators (mentioned above), and our IF proposed estimator. The latter, exhibits some bias with
small RMS. However, and surprisingly, the ad-hoc methods work better than the IF estimator! The
possible explanation for that, as appears in Sections \ref{SecIFest} and \ref{SecSimulation}, is the
lack of ``rich'' (or enough) smoothness of the CV-based estimators if compared to the
bootstrap-based estimators LOOB and LPOB. This opens the research venue in this regards to three
further main points (Section~\ref{SecDisc}).

First, more comprehensive computational study is needed to compare among the ad-hoc estimators and
the IF-based estimator to investigate when they really win or loose. Second, more mathematical
investigation is needed to know why the ad-hoc estimators work well while they ignore the covariance
structure among the CV folds. Third, for the sake of more understanding of the nature of the
problem, it is necessary to derive a mathematical expression for the ``amount of smoothness''
(measured in terms of the number of resampling permutations) necessary for decreasing the bias of
the IF method, to win over ad-hoc methods, as its bootstrap counterpart does.

\subsection{Formalization and Notation}\label{sec:form-notat}
After the verbal introduction above it is pedagogical to introduce the mathematical background and
notation before delving into the analysis in subsequent sections. As was introduced in
\cite{Yousef2019LeisurelyLookVersionsVariants-arxiv}, from which we borrow all notations and build on all
results, we always advocate for using the Area Under the ROC curve (AUC) as classifiers performance
measure for that it is prevalence (threshold) independence. We will use the following notation
\begin{subequations}
  \begin{align}
    \widehat{AUC} &  =\frac{1}{n_{1}n_{2}}\sum_{i}\sum_{j} \psi\left(  h_{\mathbf{X}}\left(x_{i}\right)  ,h_{\mathbf{X}}\left(  y_{j}\right)  \right),\label{eq:AUC}\\
    \psi(a,b)&=\left\{%
               \begin{array}[c]{ccc}
                 1 &  & a>b\\
                 1/2 &  & a=b\\
                 0 &  & a<b
               \end{array}\right.,
  \end{align}
\end{subequations}
where, $h_{\mathbf{X}}(x_i)$ is the score given to the testing observation $x_i$, by the classifier
trained on the set $\mathbf{X}$; and the testing set is composed of $n_1$ and $n_2$ observations from
both classes.

\bigskip

When there is only one available dataset, the training and testing sets can be pseudo-produced by CV
resampling mechanism, which is formalized as follows. The mapping $\mathcal{K}$, assigns each
observation to one partition. Therefore its inverse gives the set of all observations belonging to a
particular partition.  Formally, we say
$\mathcal{K}:\left\{ 1,\ldots,n\right\} \mapsto\left\{ 1,\ldots,K\right\} ,~K=n/n_{K},$ $~$such that
$\mathcal{K}\left( i\right) =k,\quad n_{K}\left( k-1\right) <i\leq n_{K}k,\quad k=1,\ldots,K$, and
therefore $\sum_{i}I_{\left( \mathcal{K}\left( i\right) =k\right) }=n_{K}~\forall k.$ Then, the
dataset that excludes the partition number $k$ is
$\mathbf{X}_{\left( \left\{ k\right\} \right) }=\left\{ x_{i}:\mathcal{K}\left( i\right) \neq
  k\right\} .$ In the special case of leave-one-out CV,
$n_{K}=1,~K=n,~\mathcal{K}\left( i\right) =i,~\mathbf{X}_{\left( \left\{ i\right\} \right) }=\left\{
  x_{i^{\prime}} :\mathcal{K}\left( i^{\prime}\right) \neq i\right\} =\left\{ x_{i^{\prime
    }}:i^{\prime}\neq i\right\} =x_{1},\ldots,x_{i-1},x_{i+1},\ldots ,x_{n}=\mathbf{X}_{\left(
    i\right) }$. And hence, we call it CVN, as opposed to CV1 as appears in some literature, just
out of obsessiveness for preserving the notational consistency. The different CV versions used by
practitioners in literature and analyzed in \cite{Yousef2019LeisurelyLookVersionsVariants-arxiv} are
the leave-one-out, $K$-fold, repeated $K$-fold, and Monte-Carlo CV. These versions were formalized
respectively as follows:
\begin{subequations}\label{eq:CVEstimators}
  \begin{align}
    \widehat{AUC}^{\left( CVN\right) } & =\frac{1}{n_{1}n_{2}}\sum_{j=1} ^{n_{2}}\sum_{i=1}^{n_{1}}\, \Biggl[ \psi\left( h(x_i), h(y_j)\right)\Biggr],\notag\\
    h& = h_{\mathbf{X1}_{\left( i\right) }\mathbf{X2}_{\left( j\right) }}\label{EQCVN}\\
    \widehat{AUC}^{\left( CVK\right) } & =\frac{1}{n_{1}n_{2}}\sum_{j=1} ^{n_{2}}\sum_{i=1}^{n_{1}}\, \Biggl[\psi\left( h(x_i), h(y_j)\right) \Biggr],\notag\\
    h&= h_{\mathbf{X1}_{\left( \left\{ \mathcal{K}_{1}(i)\right\} \right) }\mathbf{X2}_{\left( \left\{ \mathcal{K}_{2}(j)\right\} \right) }}\label{EQCVK}\\
    \widehat{AUC}^{\left( CVKR\right) } & =\frac{1}{n_{1}n_{2}}\sum _{j=1}^{n_{2}}\sum_{i=1}^{n_{1}}\left[ \left.  \sum_{m}\psi\left( h(x_i), h(y_j)\right) \right/ M\right]\notag\\
    h&= h_{\mathbf{X1}_{\left( \left\{ \mathcal{K}_{1m}(i)\right\} \right) }\mathbf{X2}_{\left( \left\{ \mathcal{K}_{2m}(j)\right\} \right) }} .\label{EQCVKR}\\
    \widehat{AUC}^{\left(  CVKM\right)  } &  =\frac{1}{n_{1}n_{2}}\sum_{j=1}^{n_{2}}\sum_{i=1}^{n_{1}}\left[  \left.  \sum_{m}I_{j}^{m}I_{i}^{m}\psi\left( h(x_i), h(y_j)\right)  \right/  \sum_{m}I_{i}^{m}I_{j}^{m}\right],\notag\\
    h&=   h_{\mathbf{X1}_{\left(\{1\},m\right)}\mathbf{X2}_{\left(  \left\{  1\right\},m\right)  }}.\label{EQCVKM}
  \end{align}%
\end{subequations}

\subsection{Manuscript Roadmap}\label{sec:manuscript-roadmap}
This article is organized as follows. In Section \ref{SecBengio} we introduce some ad-hoc methods
for estimating the variance of the CV-based estimators. We do that for the estimators of the error
rate, not the AUC, to get the insight from the simplified analysis of the former. We analyze one of
these methods, following \cite{Bengio2004NoUnbiasedEstKCV}, and derive an expression for its
bias. Analyzing the rest of them is very similar, a task that we do not pursue since all of these
estimators are variants of one method that ignores the covariance structure as mentioned above. In
Section \ref{SecIFest} we derive our IF-based estimator that estimates the variance of the CV-based
estimators. This section assumes very good familiarity with the IF approach
\citep{Hampel1974TheInfluence, Hampel1986RobustStatistics, Huber2004RobustStatistics}, along with
familiarity with the work done for estimating the variance of bootstrap-based estimators~\citep
{Efron1997ImprovementsOnCross, Yousef2005EstimatingThe}. In Section \ref{SecSimulation} we
illustrate the results of the simulation study to compare our method to four versions of the ad-hoc
method. Section \ref{SecDisc} is a description of the work suggested to complement the present
article. Finally, all detailed proofs are deferred to the \APPENDIX~to develop lucid read of the
article.


\section{Analysis of ad-hoc Methods of Estimating the variance of CVK Following \cite{Bengio2004NoUnbiasedEstKCV}}\label{SecBengio}
We will consider only the error rate in this section and analyze the variance of the error rate
CVK-based estimator to get insight from its simplified analysis if compared to the AUC. Consider the
CVK estimator of the error rate of a classifier%
\begin{equation}
  \widehat{Err}^{\left(  CVK\right)  }=\frac{1}{K}\sum_{k}\sum_{i\in
    \mathcal{K}^{-1}\left(  k\right)  }\frac{1}{n_{K}}Q\left(  x_{i}%
    ,\mathbf{X}_{\left(  \left\{  \mathcal{K}\left(  i\right)  \right\}  \right)
    }\right), \label{EQCVKpart}%
\end{equation}%
which can be rewritten as%
\begin{align}
  \widehat{Err}^{\left(  CVK\right)  }  &  =\frac{1}{K}\sum_{k}err_{k}\label{EQCVKbengio},
\end{align}
where $err_{k} =\sum_{i\in\mathcal{K}^{-1}\left( k\right) }\frac{1}{n_{K}}e_{i}$ and
$ e_{i} =Q\left( x_{i},\mathbf{X}_{\left( \left\{ \mathcal{K}\left(i\right) \right\}
    \right)}\right)$. \cite{Bengio2004NoUnbiasedEstKCV} derived a closed form expression of the
variance of (\ref{EQCVKbengio}) in terms of the covariance structure of $e_{i}$. They proved that%
\begin{subequations}
  \begin{align}
    \operatorname*{Var}\left[  \widehat{Err}^{\left(  CVK\right)  }\right]&=\frac{1}{n}\sigma^{2}+\frac{n_{K}-1}{n}\omega+\frac{n-n_{K}}{n}\gamma,\label{EQvarCV}\\
    \sigma^{2}  &  =\operatorname*{Var}e_{i},~i=1,\ldots,n;\label{EQcomp1}\\
    \omega &  =\operatorname*{Cov}\left(e_{i},e_{j}\right),~i\neq j,~\mathcal{K}\left(  i\right)  =\mathcal{K}\left(  j\right);\label{EQcomp2}\\
    \gamma &  =\operatorname*{Cov}\left(  e_{i},e_{j}\right)  ,~i\neq j,~\mathcal{K}\left(  i\right)  \neq\mathcal{K}\left(  j\right)  .\label{EQcomp3}
  \end{align}
\end{subequations}%
The components (\ref{EQcomp1})--(\ref{EQcomp3}) are the three possible components of a covariance
structure. This arises from the symmetry of all the observations: we either have a single variance
(\ref{EQcomp1}), a covariance between the errors of two observations belonging to the same testing
fold (\ref{EQcomp2}), or a covariance between the errors of two observations which do not belong to
the same testing fold (\ref{EQcomp3}). The reader should notice that the condition $i\neq j$ in
(\ref{EQcomp3}) is redundant since
$\mathcal{K}\left( i\right) \neq\mathcal{K}\left( j\right) \rightarrow i\neq j$. The estimate%
\begin{equation}
  \widehat{\operatorname*{Var}}\left[  \widehat{Err}^{\left(  CVK\right)}\right]  =\frac{1}{K}\left[  \frac{1}{K-1}\sum_{k=1}^{K}\left(  err_{k}-\frac{1}{K}\sum_{k^{\prime}=1}^{K}err_{k^{\prime}}\right)  ^{2}\right]\label{EQnaiveEst}%
\end{equation}
is one version of what many people use for estimating the variance of the CVK.  This is the same
form of the very well known UMVU estimator of the variance but it is not the UMVUE since the $K$
values of $err_{k}$ are not independent. This is obvious since every pair
$err_{i},~err_{j},~i\neq j$, share $K-1$ folds in the training set if they belong to the same testing
fold, and share $K-2$ if they belong to two different testing folds. It is worth mentioning that
some experts in the field, e.g., \cite{Hastie2001TheElements}, deliberately use (\ref{EQnaiveEst})
in selecting classifiers in the design stage based on their relative values of (\ref{EQnaiveEst}),
not to give a rigorous estimate.

\begin{lemma}\label{LemmaBiasOfVar}
  The Bias of the variance estimator~\eqref{EQnaiveEst} is given by:%
  \begin{align}
    \operatorname*{Bias}\left(  \widehat{\operatorname*{Var}}\left[  \widehat{Err}^{\left(CVK\right)  }\right]  \right) &  =-\gamma.\label{EqBiasOfVar}
  \end{align}
\end{lemma}%
Recall that $\gamma$ is the covariance between the errors of two observations, which do not belong
to the same block (fold). Therefore, it is anticipated that $\gamma$ has a positive value since the
training sets of these two observations will be similar to some extent, which should impose the same
classifier behavior (on average) on both observations.
More details on the behavior of $\gamma$ with respect to sample size and number of folds $K$ can
be found in the simulation study in \cite{Bengio2004NoUnbiasedEstKCV}, which suggests that
$\gamma$ is inversely proportional to $n$.
It is quite reasonable to conjecture that the ad-hoc estimator exhibits a downward bias, and this
bias vanishes with increasing the data size. Our simulations in Section \ref{SecSimulation}
supports this conjecture.

The analysis above considered one version of the naive variance estimate of the error rate CV
estimator. When CVKR is used, we have $R$ repetitions of the CVK; each can provide a variance
estimate of the form (\ref{EQnaiveEst}). Averaging those estimates over the $R$ repetitions is
another version. Both versions are examined in our simulations in
Section~\ref{SecSimulation}. Following the same route above for the AUC is straight forward.

\bigskip

In case of the AUC, a two sample statistic \citep{Randles1979IntroductionTo}, one can make up
different forms of the ad-hoc estimate above:
\begin{subequations}\label{EQadhocVersions}%
  \begin{align}
    \widehat{\operatorname*{Var}}_{1}^{CVK}  &  =\frac{1}{\sqrt{K_{1}K_{2}}}\cdot \frac{1}{K_{1}K_{2}-1} \times\notag\\
                                           &\sum_{k=1}^{K_{1}K_{2}}\left(  AUC_{k_{1}k_{2}}-\frac{1}{K_{1}K_{2}}\sum_{k=1}^{K_{1}K_{2}}AUC_{k_{1}k_{2}}\right)^{2}  ,\\
    \widehat{\operatorname*{Var}}_{2}^{CVK}  &  =\frac{1}{K}\left[  \frac{1}{K-1}\sum_{k=1}^{K}\left(  AUC_{k}-\frac{1}{K}\sum_{k=1}^{K}AUC_{k}\right)^{2}\right]  ,\label{EQadhocCVK2}\\
    \widehat{\operatorname*{Var}}_{3}^{CVK}  &  =c_{1}\sum_{k_{1}=1}^{K_{1}}\left(  \frac{1}{K_{2}}\sum_{k_{2}=1}^{K_{2}}AUC_{k_{1}k_{2}}-\widehat{AUC}^{CVK}\right)  ^{2}+\notag\\
                                           &c_{2}\sum_{k_{2}=1}^{K_{2}}\left(  \frac{1}{K_{1}}\sum_{k_{1}=1}^{K_{1}}AUC_{k_{1}k_{2}}-\widehat{AUC}^{CVK}\right)  ^{2},\label{EQadhocCVK}%
  \end{align}
\end{subequations}
where, $AUC_{k_{1}k_{2}}$ is the AUC of the testing folds $k_{1},~k_{2}$ (where the classifier is
trained on the other folds), $c_{1}=\frac{1}%
{K_{1}\left( K_{1}-1\right) }$ for unbiasedness or $\frac{1}{K_{1}^{2}}$for MLE criterion, and
$c_{2}$ is defined analogously, $AUC_{k}$ is the AUC of the the testing folds $k_{1}=k_{2}=k$
(where the classifier is trained on the others folds), and $\widehat{AUC}^{CVK}$ is the usual
CVK-based AUC estimator~\eqref{EQCVK}.

The estimator $\widehat{\operatorname*{Var}}_{2}^{CVK}$ is suitable for the case of imposing testing
on those folds with similar index, i.e., $K_{1}%
=K_{2}=K$. The estimator $\widehat{\operatorname*{Var}}_{1}^{CVK}$ has the same spirt of
$\widehat{\operatorname*{Var}}_{2}^{CVK}$, i.e., the variance of the mean of independent
observations is the variance of an observation divided by the number of observations ($K$ in our
case). This is clear in $\widehat{\operatorname*{Var}}_{2}^{CVK}$, where $AUC_{k},~k=1,\ldots,K$,
are assumed independent (which is a wrong assumption; this is why it is an ad hoc
estimator). However, in $\widehat{\operatorname*{Var}}_{1}^{CVK}$ the variance of $AUC_{k_{1}k_{2}}$
is loosely estimated from the pool of $K_{1}K_{2}$ observations (also, assumed independent); and the
term $1/\sqrt{K_{1}K_{2}}$ is the ad hoc analogue of $1/K$. Neverthless, both
$\widehat {\operatorname*{Var}}_{1}^{CVK}$ and $\widehat{\operatorname*{Var}}_{2}^{CVK}$ give very
similar results (Section~\ref{SecSimulation}).

\bigskip

In the case of CVKR \eqref{EQCVKR}, one can define another ad-hoc estimator for every estimator of
the three above. E.g., in terms of (\ref{EQadhocCVK}) a practitioner can define.
\begin{equation}
  \widehat{\operatorname*{Var}}_{3}^{CVKR}=\frac{1}{R}\sum_{r=1}^{R}%
  \widehat{\operatorname*{Var}}_{3}^{CVK}\left(  r\right)  . \label{EQadhocCVKR}%
\end{equation}
For the case of CVKM \eqref{EQCVKM} the ad-hoc estimate is%
\begin{equation}
  \widehat{\operatorname*{Var}}^{CVKM}=\frac{1}{\sqrt{K_{1}K_{2}}}\left[
    \frac{1}{M-1}\sum_{m=1}^{M}\left(  AUC_{11m}-\frac{1}{M}\sum_{m=1}%
      ^{M}AUC_{11m}\right)  ^{2}\right]  , \label{Eqadhoccvkm}%
\end{equation}
where $AUC_{11m}$ is the AUC from the only testing fold in each repetition, since we only test on
$k_{1}=k_{2}=1~\forall m$ as defined in \cite{Yousef2019LeisurelyLookVersionsVariants-arxiv}. This
estimate gives very similar results to the estimates above.


\section{Estimating Cross Validation Variability}\label{SecIFest}
In this section we derive our novel estimator that estimates the cross validation variability. As
was introduced above, we assume the reader to be comfortable with the IF approach. Section
\ref{SecWhichVersion} is an overview on CV estimators and a discussion on which estimator among them
is suitable for variance estimation; we will conclude that CVKM and CVKR are the only
candidates. Section \ref{SecVarEst} is analysis and derivation. We demonstrate our analysis in this
section in terms of the AUC as a performance measure, and for lucidity of the article we defer all
proofs to the \APPENDIX.

\subsection{Which Version of Cross Validation?}\label{SecWhichVersion}
Recall the different versions of the CV~\eqref{eq:CVEstimators}. \cite{Efron1997ImprovementsOnCross}
alluded to the fact that cross validation is not a \textquotedblleft smooth
statistic\textquotedblright. They derived an IF estimator for estimating the variance of the Leave
One Out Bootstrap (LOOB) estimator instead, since the latter is a smooth statistic.
\cite{Yousef2005EstimatingThe} followed the same route for deriving an IF estimator for estimating
the variance of the Leave Pair Out Bootstrap (LPOB) estimator since it is a smooth statistic as
well. Several years, now, have passed and we come back to estimating the variance of the CV
estimators.

\bigskip

In the vernacular, the smoothness property means that a little variation in one observation should
result in a little variation in the statistic. Should the reader need more formalisms one of the
following should be visited: \cite{Hampel1974TheInfluence, Hampel1986RobustStatistics,
  Huber2004RobustStatistics}; more intuitive and experimental explanation is provided in
\cite{Yousef2019AucSmoothness-arxiv}. The CVN and CVK are not smooth statistics; any small change in one
observation does not lead to small change in the kernel $\psi$ inside the two summations. This is
true since the value of $\psi$ is either 0, 0.5, or 1. However, the other two versions, CVKM and
CVKR are happily smooth. The performance of each observation (or pair of observations) is averaged
over many training sets. This averaging is done by the summation $\Sigma_{m}$. This averaging
smoothes the quantity inside the square brackets and makes it suitable for differentiation. We
proved in \cite{Yousef2019LeisurelyLookVersionsVariants-arxiv} that both CVKM and CVKR are identical
asymptotically; i.e., CVKM converges almost surely to CVKR. Since both are smooth estimators, both
are candidates for the IF approach with the difference of how they are formed and the difference in
the number of trials $R$ and $M$. We explained in \cite{Yousef2019LeisurelyLookVersionsVariants-arxiv} why
$M$ should be larger than $R$ for the same required accuracy. We will give an IF-based estimation of
the variance of both CVKM and CVKR just for completion.  However, it should be obvious that both
should be identical asymptotically with large values of $R$ and $M$.

It is worth mentioning that CVKM and CVKR are smooth statistics because we wrote them the way it
appears in (\ref{EQCVKM}) and (\ref{EQCVKR}). CVKR can be written with swapping the summation over
$m$ with both the summations over $i$ and $j$. This other way is exactly equivalent (in value) to
the version above. However it is not equivalent from the sense of smoothness; it is unsmooth
(since, for every testing observation we do not average over many training sets) and therefore not
suitable for the IF estimation. The same comment is applicable to CVKM. This situation is very
similar to the two versions of the bootstrap estimator, the star vs. the leave-pair-out
\citep{Yousef2005EstimatingThe, Yousef2019LeisurelyLookVersionsVariants-arxiv}. We start with estimating
the variance of CVKM for that it is much easier than estimating the variance of CVKR.

\subsection{Estimating the Variance of $\widehat{AUC}^{( CVKM) }$}\label{SecVarEst}
For short, we can change the notation of (\ref{EQCVKM}) a little bit by dropping all the subscripts
indicating the training datasets, using $j_{1}$ and $j_{2}$ in place of $i$ and $j$ respectively,
and using the indecies $j_{1}$ and $j_{2}$ directly to indicate the testing observations
$x_{j_{1}}$and $x_{j_{2}}$respectively. Then we can write%
\begin{align}
  &\widehat{AUC}^{\left(  CVKM\right)  }=\frac{1}{n_{1}n_{2}}\times\notag\\
  &\sum_{j_{2}=1}^{n_{2}}\sum_{j_{1}=1}^{n_{1}}\left[  \left.  \sum_{m}I_{j_{2}}^{m}I_{j_{1}}^{m}\psi\left(  h_{m}\left(  j_{1}\right)  ,h_{m}\left(j_{2}\right)  \right)  \right/  \sum_{m}I_{j_{2}}^{m}I_{j_{1}}^{m}\right].
\end{align}
We can follow a very similar route to that of \cite{Yousef2005EstimatingThe}, when the LPOB
estimator was treated, but with taking care of the difference in resampling mechanism between the
CVKM and the bootstrap. If the probability of the $i\textsuperscript{th}$ observation is perturbed by an amount $\varepsilon$ then
\begin{align}
  \widehat{AUC}_{\varepsilon,i}^{\left(  CVKM\right)  }  &  =\sum\limits_{j_{2}=1}^{n_{2}}{\sum\limits_{j_{1}=1}^{n_{1}}{\hat{f}_{1_{\varepsilon,i}}(j_{1})\hat{f}_{2_{\varepsilon,i}}(j_{2})}}\frac{\sum_{m}I_{j_{2}}^{m}I_{j_{1}}^{m}\psi\left(  h_{m}\left(  j_{1}\right)  ,h_{m}\left(j_{2}\right)  \right)  g_{\varepsilon,i}}{\sum_{m}I_{j_{2}}^{m}I_{j_{1}}^{m}g_{\varepsilon,i}}\notag\\
                                                       &  =\sum\limits_{j_{2}=1}^{n_{2}}{\sum\limits_{j_{1}=1}^{n_{1}}{A}}\left(\varepsilon\right)  \frac{B\left(  \varepsilon\right)  }{C\left(\varepsilon\right)  },\notag
\end{align}
where
${{A}}\left( \varepsilon\right)
={{\hat{f}_{1_{\varepsilon,i}}(j_{1})\hat{f}_{2_{\varepsilon,i}}(j_{2})}}$,
${{B}}\left( \varepsilon\right)=\sum_{m}I_{j_{2}}^{m}I_{j_{1}}^{m}\psi\left( h_{m}\left(
    j_{1}\right),h_{m}\left( j_{2}\right) \right)g_{\varepsilon,i}$, and
$ {{C}}\left( \varepsilon\right) =\sum_{m}I_{j_{2}}^{m}I_{j_{1}}^{m}g_{\varepsilon,i}$; and the
empirical probability after perturbing an observation $i$ from each class is given by%
\begin{align}
  \widehat{f}_{k_{\varepsilon,i}}(j_{k})  &  =\left\{%
                                         \begin{array}[c]{lr}%
                                           \frac{1-\varepsilon}{n_{k}}+\varepsilon, & j_{k}=i\\
                                           \frac{1-\varepsilon}{n_{k}}, & j_{k}\neq i
                                         \end{array}\right., i,j_{k}=1,\ldots,n_{1},k=1,2.\notag
\end{align}
and their derivatives are given by
\begin{align}
  \frac{\partial f_{k_{\varepsilon,i}}(j_{k})}{\partial\varepsilon}  & =\delta_{ij_{k}}-1/n_{k},\ k=1,2.
\end{align}

Regarding $g_{\varepsilon,i}$, it is a shorthand writing for
$g_{\varepsilon,i}\left( j_{1},j_{2},m\right) $, which is the probability of the training set that
excludes the two folds of the observations $j_{1}$ and $j_{2}$ in the repetition $m$ after
perturbing the observation $i$ with a probability measure $\varepsilon$. Suppose that we perturb the
observation $i\in\mathbf{X}_{1}$; then since the partitioning of the two sets is done independently,
we can write
$g_{\varepsilon,i}\left( j_{1},j_{2},m\right)
=g_{1\varepsilon,i}\left(j_{1},m\right)g_{2\varepsilon,i}\left( j_{2},m\right)$; and when there is
no perturbation, i.e., $\varepsilon=0$, $g_{\varepsilon,i}\left( j_{1},j_{2},m\right) $ becomes
$g_{0} =g_{10}\cdot g_{20} =\frac{1}{\binom{n_{1}}{n_{1K}}}\cdot\frac{1}{\binom{n_{2}}{n_{2K}}}$,
where $n_{1K}$ and $n_{2K}$ are the number of observations in each fold. Now, we have to find
$g_{1\varepsilon,i}\left( j_{1},m\right)$ and its derivative; both of which are provided in the next
lemma. The proof is deferred to the \APPENDIX, where, in addition, we provide another
method of deriving $g_{_{1}\varepsilon,i}$ and its derivative for more understanding of the
probability perturbation problem. This other method relies on a different perturbation approach but
gives very similar results to the following lemma; however, it lacks a closed form expression.
\begin{lemma}\label{lem:prob}The probability, and its derivative, of the training sets of the CV under the
  probability perturbation is given by:%
  \begin{align}
    g_{_{1}\varepsilon,i}&  =\frac{1}{(  n_{1}-n_{1K})  \binom{n_{1}}{n_{1}-n_{1K}}}\sum_{r=1}^{n_{1}-n_{1K}}(  1-\varepsilon)  ^{(  r-1)}(  \varepsilon n_{1}-r\varepsilon+1).\notag\\
    \frac{g_{1\varepsilon,i}^{\cdot}\left(  0\right)  }{g_{10}}  & =n_{1K}-I_{i}^{m}n_{1}, \label{EQgdot}%
  \end{align}
\end{lemma}
It is interesting, as well, to rewrite (\ref{EQgdot}) as
\begin{align}
  \frac{g_{1\varepsilon,i}^{\cdot}\left(  0\right)  }{g_{10}} &=n_{1}\left(N_{i}^{m}-1+\frac{1}{K}\right)  ,
\end{align}
where $N_{i}^{m}~\left( =1-I_{i}^{m}\right) $ is the number of times the observation $i$ appears in the
training set of the repetition $m$. This is similar to the derivative of the bootstrap training set
probabilities \citep[see][]{Yousef2005EstimatingThe}, with the exception that $N_{i}^{b}$ can take
any value between $0$ and $n_{1}$ due to the resampling mechanism of the bootstrap. However, for the
CV based estimators $N_{i}^{m}$ is either 1 or 0. Now, the previous mathematical introduction is sufficient to state the main lemma of the article
\begin{lemma}\label{LemmaSD}The estimation of the standard error of $  \widehat{AUC}^{\left(  CVKM\right)  }$ is given by:
  \begin{align}
    &\widehat{SD}\left[  \widehat{AUC}^{\left(  CVKM\right)  }\right]  =\sqrt{\frac{1}{n_{1}^{2}}\sum\limits_{i=1}^{n_{1}}{\hat{U}_{1_{i}}^{2}}+\frac{1}{n_{2}^{2}}\sum\limits_{j=1}^{n_{2}}{\hat{U}_{2_{j}}^{2}}},\\
    &\hat{U}_{1_{i}} = I + II + III,\\
    I  & =\widehat{AUC}_{1i}-\widehat{AUC}^{\left(  CVKM\right)  },\label{EQAUCi}\\
    \widehat{AUC}_{1i}  &  =\frac{1}{n_{2}}\sum\limits_{j_{2}=1}^{n_{2}}\left[\left.  \sum_{m}I_{j_{2}}^{m}I_{i}^{m}\psi\left(  h_{m}\left(  i\right),h_{m}\left(  j_{2}\right)  \right)  \right/  \sum_{m}I_{j_{2}}^{m}I_{i}^{m}\right] ,\notag\\
    II  &=\frac{1}{n_{1}n_{2}}\sum\limits_{j_2=1}^{n_2} {\sum\limits_{j_1=1}^{n_1}}\left[  \frac{\sum_mI_{j_2}^m I_{j_1}^m \psi\left( h_m(j_1)  ,h_{m}(j_2)  \right)  \frac{g_{\varepsilon,i}^{\cdot}(0)}{g(0)} } {\sum_mI_{j_2}^mI_{j_1}^m}\right]\notag\\
    III&  =\frac{1}{n_{1}n_{2}}\sum\limits_{j_{2}=1}^{n_{2}}{\sum\limits_{j_{1}=1}^{n_{1}}}\frac{\left(  \sum_{m}I_{j_{2}}^{m}I_{j_{1}}^{m}\psi\left(h_{m}\left(  j_{1}\right)  ,h_{m}\left(  j_{2}\right)  \right)  \right)\left(  \sum_{m}I_{j_{2}}^{m}I_{j_{1}}^{m}\frac{g_{\varepsilon,i}^{\cdot}\left(  0\right)  }{g\left(  0\right)  }\right)  }{\left(  \sum_{m}I_{j_{2}}^{m}I_{j_{1}}^{m}\right)  ^{2}},\notag
  \end{align}
  and similarly, $\hat{U}_{2_{j}},~j=1,\ldots n_{2}$, is given.
\end{lemma}

In what follows, we treat a special case of the CVKM where $K=n$; this is the case where CVKM
converges almost surely to CVN. Interestingly, for this case the ratio (\ref{EQgdot}) can be
expressed as $\frac{g_{\varepsilon,i}^{\cdot}\left( 0\right) }{g_{0}}=n_{1K}
-\delta_{ij_{1}}n_{1}$. Then,
\begin{align}
  II  &  =n_{1K}\widehat{AUC}^{\left(  CVKM\right)  }-\widehat{AUC}_{1i}\\
  III  &  =n_{1K}\widehat{AUC}^{\left(  CVKM\right)  }-\widehat{AUC}_{1i},
\end{align}
which means
\begin{equation}
  \hat{U}_{1_{i}}=\widehat{AUC}_{1i}-\widehat{AUC}^{\left(  CVKM\right)  },
\end{equation}
and
\begin{align}
  &\widehat{SD}\left[  \widehat{AUC}^{\left(  CVKM\right)  }\right]  =\notag\\
  &\sqrt{\frac{1}{n_{1}^{2}}\sum\limits_{i=1}^{n_{1}}\left(  \widehat{AUC}_{1i}-\widehat{AUC}^{\left(  CVKM\right)  }\right)  ^{2}+\frac{1}{n_{2}^{2}}\sum\limits_{j=1}^{n_{2}}\left(  \widehat{AUC}_{1j}-\widehat{AUC}^{\left(CVKM\right)  }\right)  ^{2}}. \label{EQIFCVN}
\end{align}
We can write $\widehat{AUC}^{\left( CVKM\right) }$ in terms of $\widehat{AUC}_{1i}$ as
\begin{align}
  \widehat{AUC}^{\left(  CVKM\right)  } &  =\frac{1}{n_{1}}{\sum\limits_{i=1}^{n_{1}}}\widehat{AUC}_{1i}.\label{EqAUCCVKMi}
\end{align}
In the same time, for CVN there is only one repetition $m$ for which $i_{i}^{m}$ and $I_{j_{2}}^{m}$ in
(\ref{EQAUCi}) equal 1. Therefore, for CVN, (\ref{EQAUCi}) reduces (with large $M$) to
\begin{align}
  \widehat{AUC}_{1i}&=\frac{1}{n_{2}}\sum\limits_{j_{2}=1}^{n_{2}}\psi\left(h\left(  i\right)  ,h\left(  j_{2}\right)  \right).
\end{align}

To get the insight from the above, imagine that we had treated the error rate, rather than the
AUC. We would get $\widehat{Err}_{i}=Q\left( i\right) $, and
$\widehat{SD}=\sqrt{\frac{1}{n_{1}^{2}}\sum\limits_{i=1}^{n_{1}}\left( Q\left(
      i\right)-\overline{Q}\right)^{2}}$. IF method for the case of CVN, therefore, simply reduces
to the naive sample variance estimate of the sample mean, the same ad-hoc estimators, which assume
that the test values are independent. This will introduce bias. The reason is that the CVN is not a
smooth statistic; and the smoothing over the summation $\Sigma_{m}$ is a fake one since for every
left-one-out observation $j_{1}$ there is only one training set consisting of the remaining
$n_{1}-1$ observations. Repeating the estimation $M$ times accounts for reproducing the same result
for every observation $j_{1}$. The opposite is true for the case $K=n/2$, for which the maximum
number of distinct $M$ trials is possible. It is known that $\binom{n}{n/K}$ is maximized when
$K=2$. Therefore, the IF for CVK when $K=n/2$ produces more accurate results than other values of
$K.$ However, all almost produce small bias, except for the extreme case when $K=n$, i.e., CVN.  For
$K=n/2$ we have the best estimate of variance; however, the estimate of the AUC itself is biased
since we train on half of the observations. One can defend that by saying this is almost the same
size on which the bootstrap is supported. Figure~\ref{FIGCVBSratio} (\APPENDIX) is a plot of
$r=\binom{n}{n/2}/n^{n}$, which is the ratio between the number of permutations of both KCV (where
$K=2$) and the Bootstrap.

\subsection{Estimating the Variance of $\widehat{AUC}^{\left( CVKR\right) }$}\label{sec:estim-vari-wideh}
For short, we can rewrite $\widehat{AUC}^{\left(  CVKR\right)  }$ (\ref{EQCVKR}) and $\widehat{AUC}_{\varepsilon,i}^{\left(  CVKR\right)  }$ as%
\begin{align}
  \widehat{AUC}^{\left(  CVKR\right)  }  &  =\frac{1}{n_{1}n_{2}}\sum_{j_{2}=1}^{n_{2}}\sum_{j_{1}=1}^{n_{1}}\left[  \left.  \sum_{m}\psi\left(h_{m}\left(  j_{1}\right)  ,h_{m}\left(  j_{2}\right)  \right)  \right/M\right]  ,\notag\\
  \widehat{AUC}_{\varepsilon,i}^{\left(  CVKR\right)  }  &  =\sum\limits_{j_{2}=1}^{n_{2}}{\sum\limits_{j_{1}=1}^{n_{1}}{\hat{f}_{1_{\varepsilon,i}}(j_{1})\hat{f}_{2_{\varepsilon,i}}(j_{2})}}\left[  \sum_{m}\psi\left(h_{m}\left(  j_{1}\right)  ,h_{m}\left(  j_{2}\right)  \right)  G_{\varepsilon,i}\right]  ,\notag
\end{align}
Then, by taking the derivatives and following the same route above we get%
\begin{subequations}\label{eq:9}
  \begin{align}
    \hat{U}_{1_{i}}&  =\widehat{AUC}_{1i}-\widehat{AUC}^{\left(  CVKR\right)  }+ \notag\\
                 &\frac{1}{n_{1}n_{2}}\sum\limits_{j_{2}=1}^{n_{2}}{\sum\limits_{j_{1}=1}^{n_{1}}}\sum_{m}\psi\left(  h_{m}\left(  j_{1}\right)  ,h_{m}\left(  j_{2}\right)\right)  \left.  G_{\varepsilon,i}^{\cdot}\right\vert _{\varepsilon=0},\\
    \widehat{AUC}_{1i}  &  =\frac{1}{n_{2}}\sum\limits_{j_{2}=1}^{n_{2}}\left[\left.  \sum_{m}\psi\left(  h_{m}\left(  i\right)  ,h_{m}\left(  j_{2}\right)\right)  \right/  M\right]  ,\\
    G_{0}  &  =\frac{1}{\prod\limits_{k=0}^{K_{2}-1}\binom{n_{2}-kn_{2K}}{n_{2K}}}\cdot\frac{1}{\prod\limits_{k=0}^{K_{1}-1}\binom{n_{1}-kn_{1K}}{n_{1K}}}.
  \end{align}
\end{subequations}
The probability $g_{\varepsilon,i}$ depends on the partition $k$ in which $j_{1}$appears in. Deriving
a closed form expression for the probability $G_{\varepsilon,i}$ is deferred to future work (Section
\ref{SecDisc}).


\begin{figure*}[!tbh]\centering
  {\scriptsize \hfil; Bias of IF\hspace{1.6cm}Bias of CVKR\hspace{1.6cm}SD of IF\hspace{1.6cm}SD of CVKR\hspace{1.6cm}RMS of IF\hspace{1.6cm}RMS of CVKR\hfil}
  \includegraphics[width=0.16\textwidth]{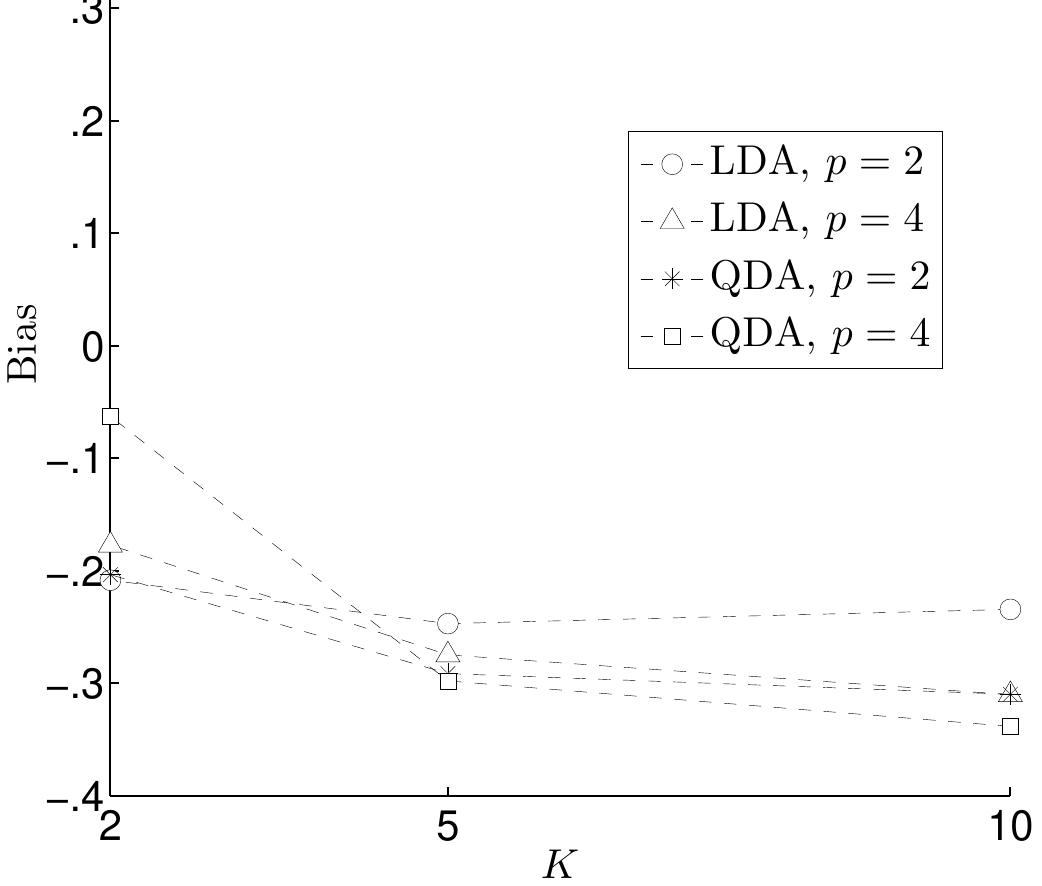}\includegraphics[width=0.16\textwidth]{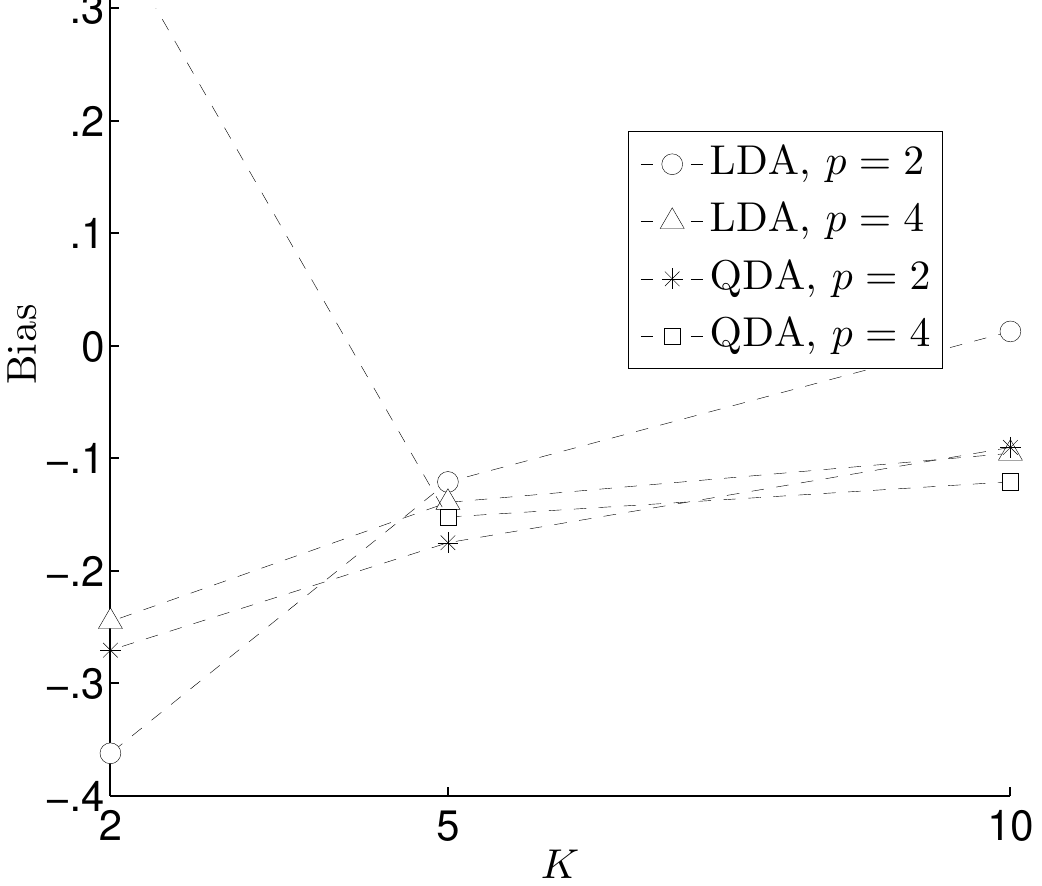}\includegraphics[width=0.16\textwidth]{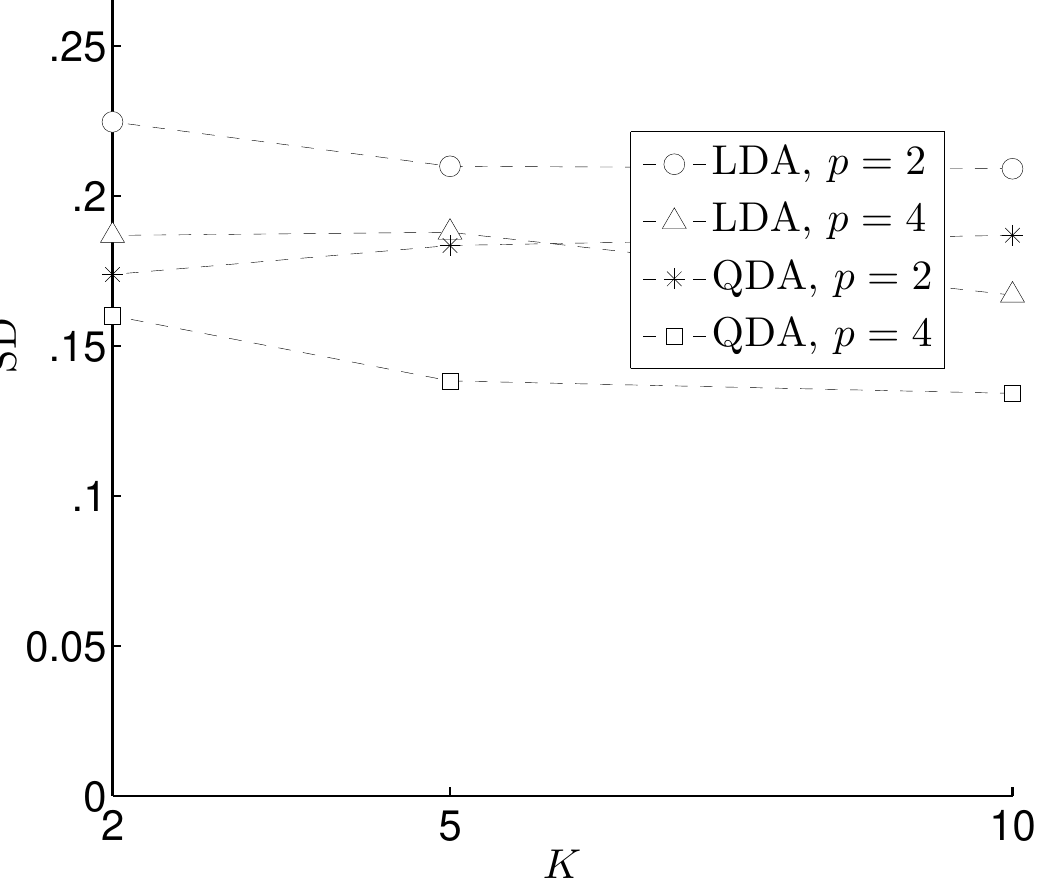}\includegraphics[width=0.16\textwidth]{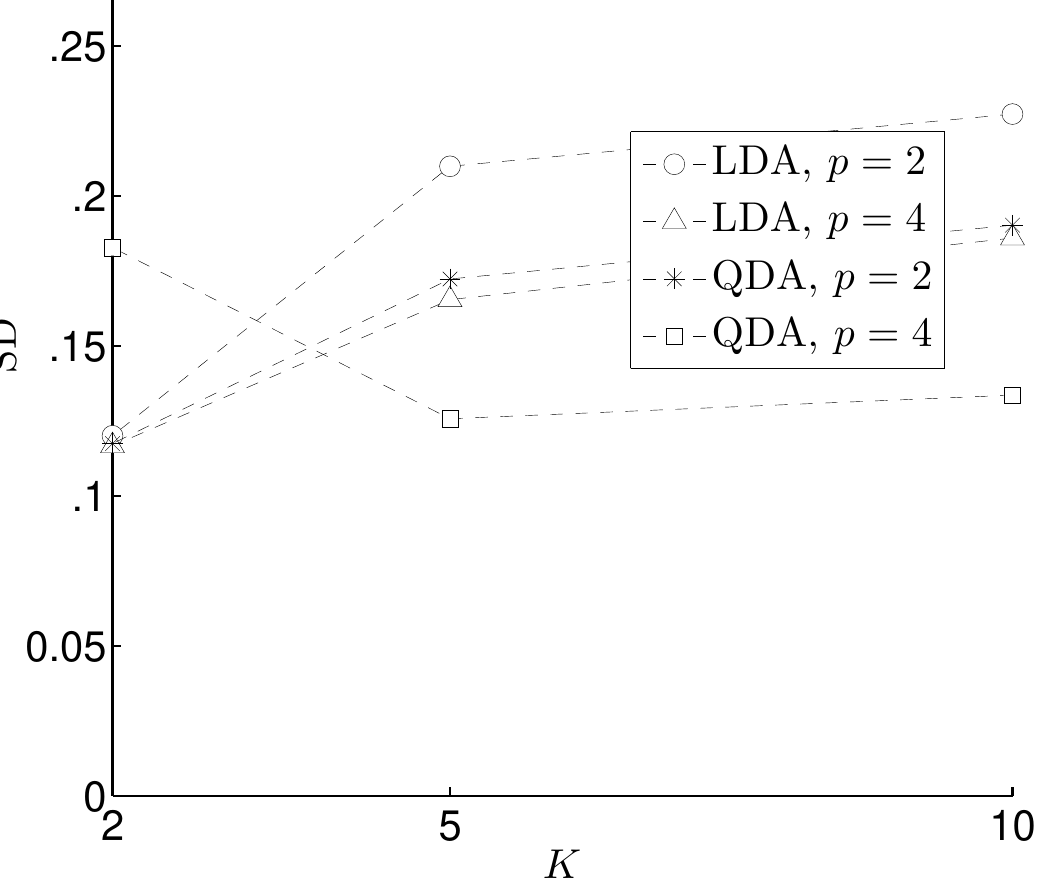}\includegraphics[width=0.16\textwidth]{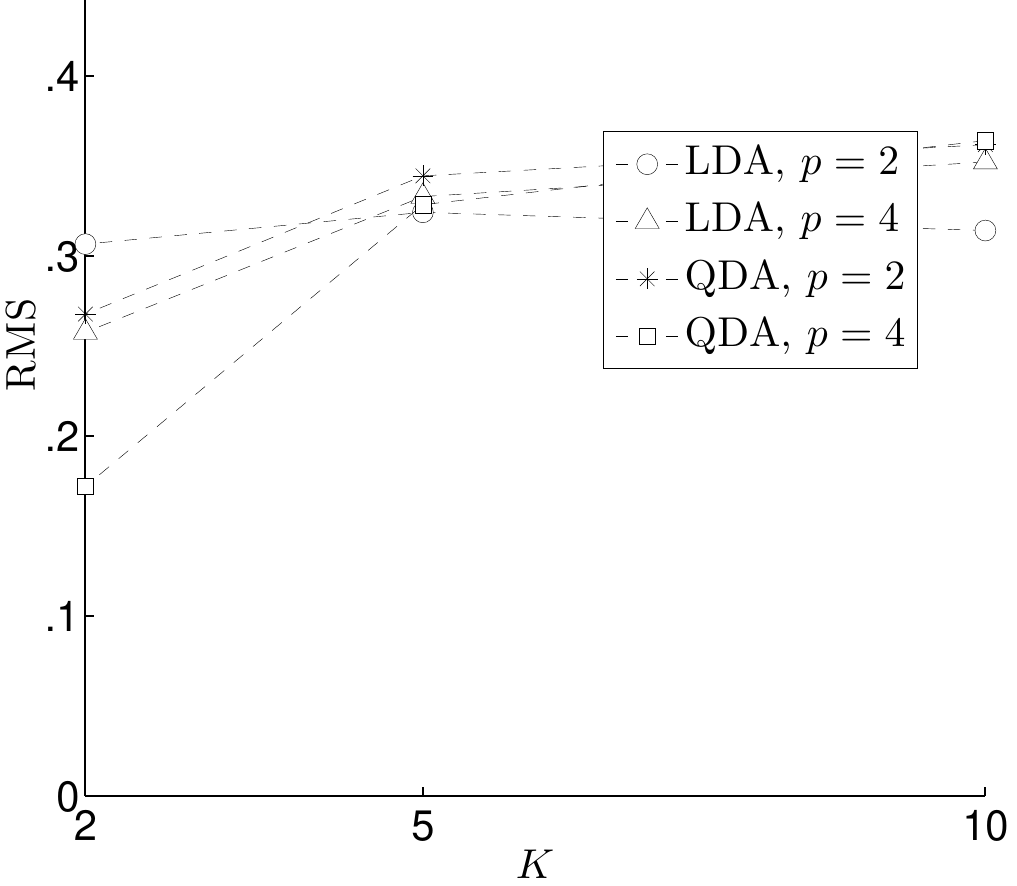}\includegraphics[width=0.16\textwidth]{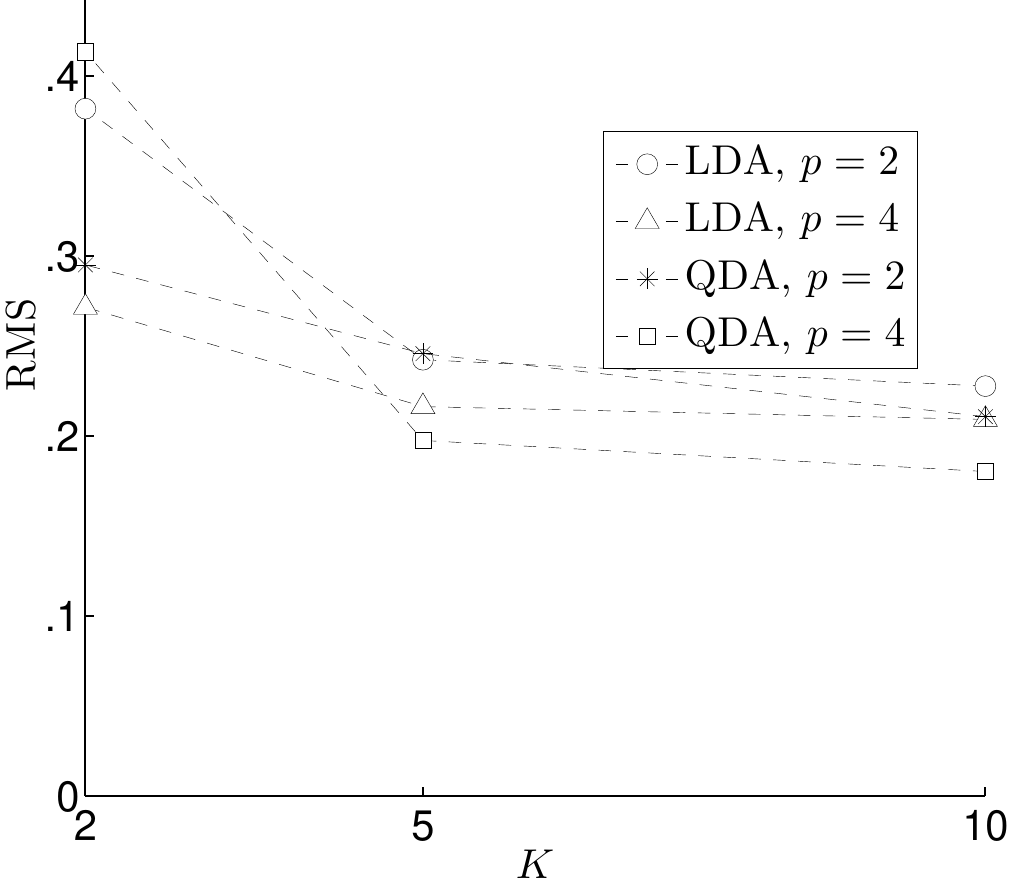}

  \includegraphics[width=0.16\textwidth]{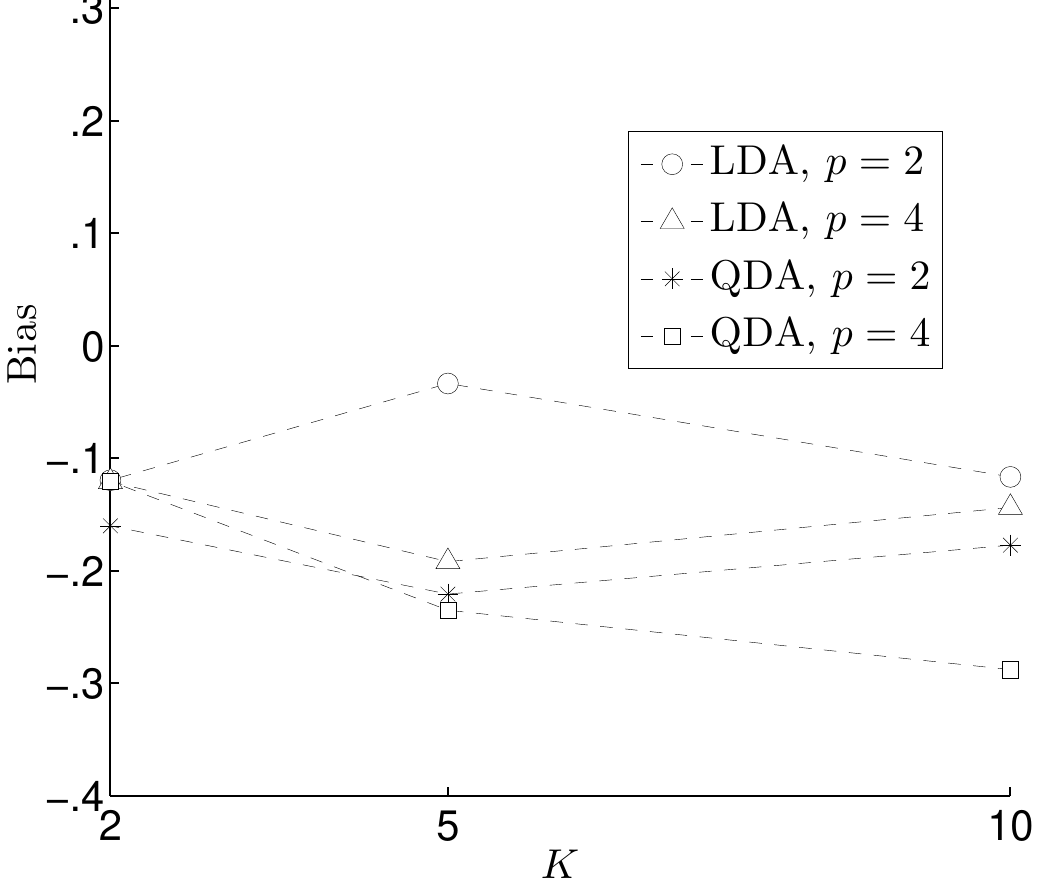}\includegraphics[width=0.16\textwidth]{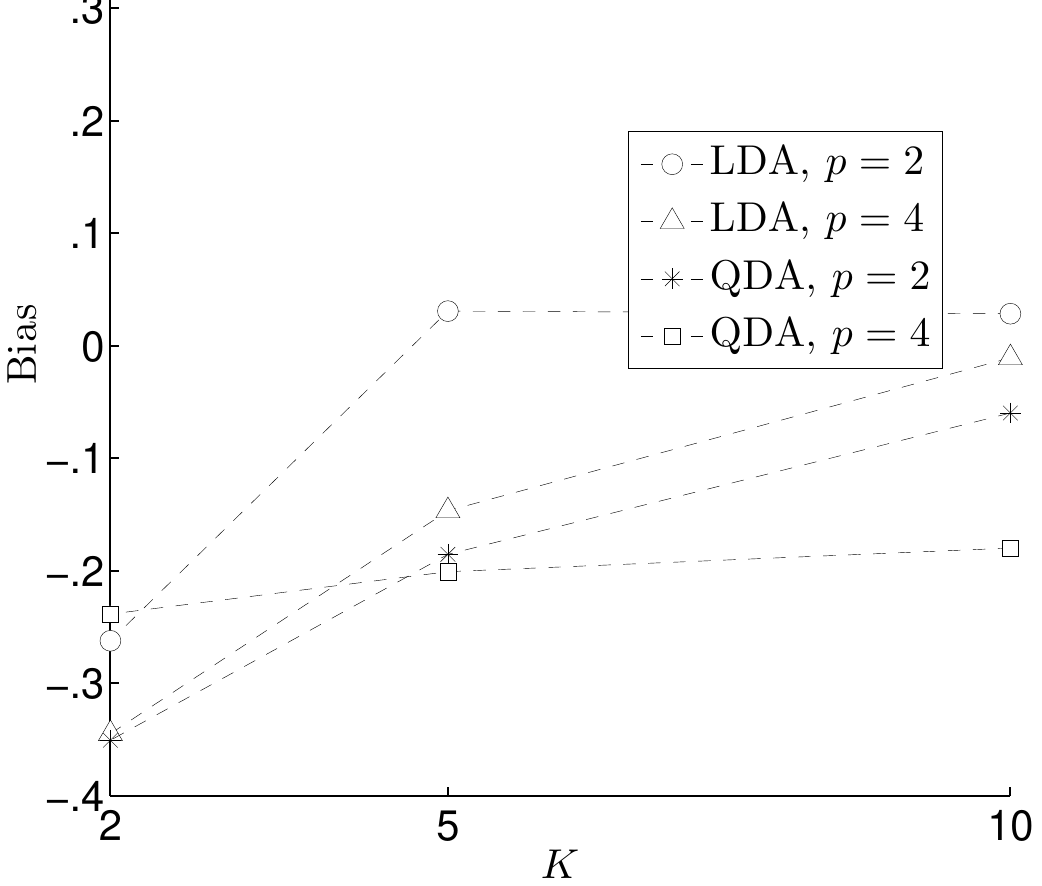}\includegraphics[width=0.16\textwidth]{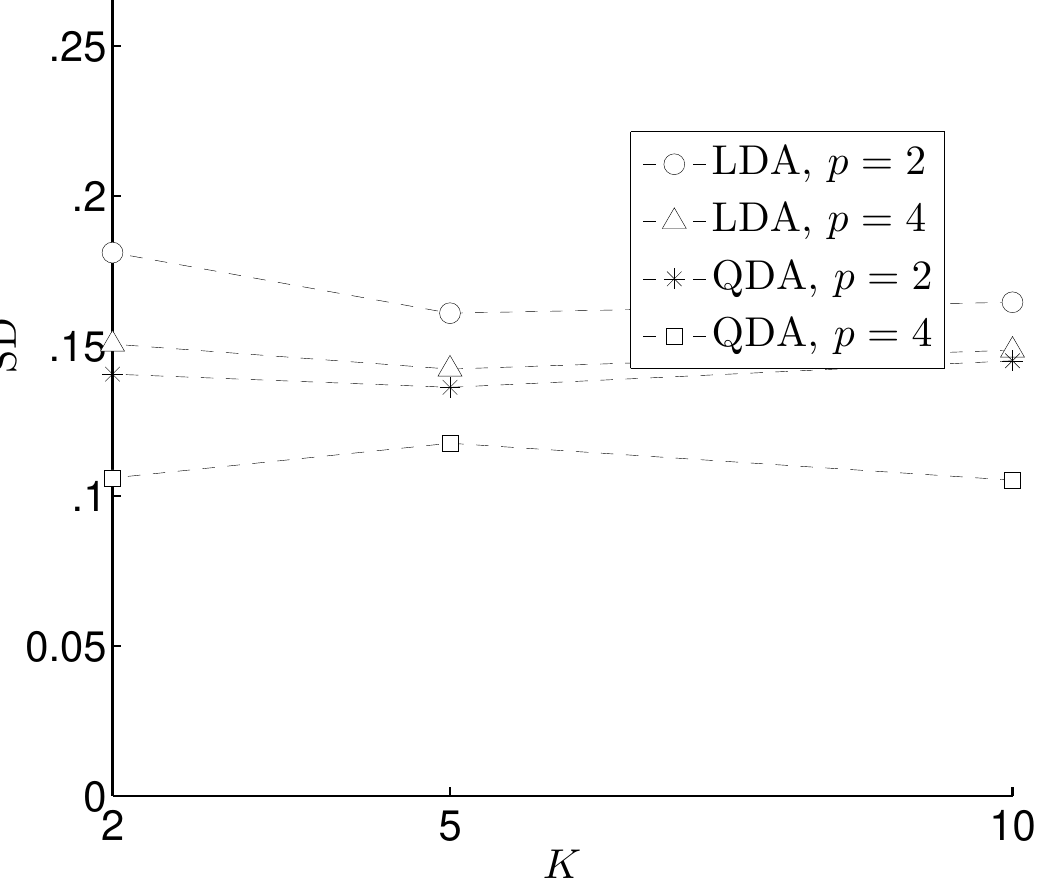}\includegraphics[width=0.16\textwidth]{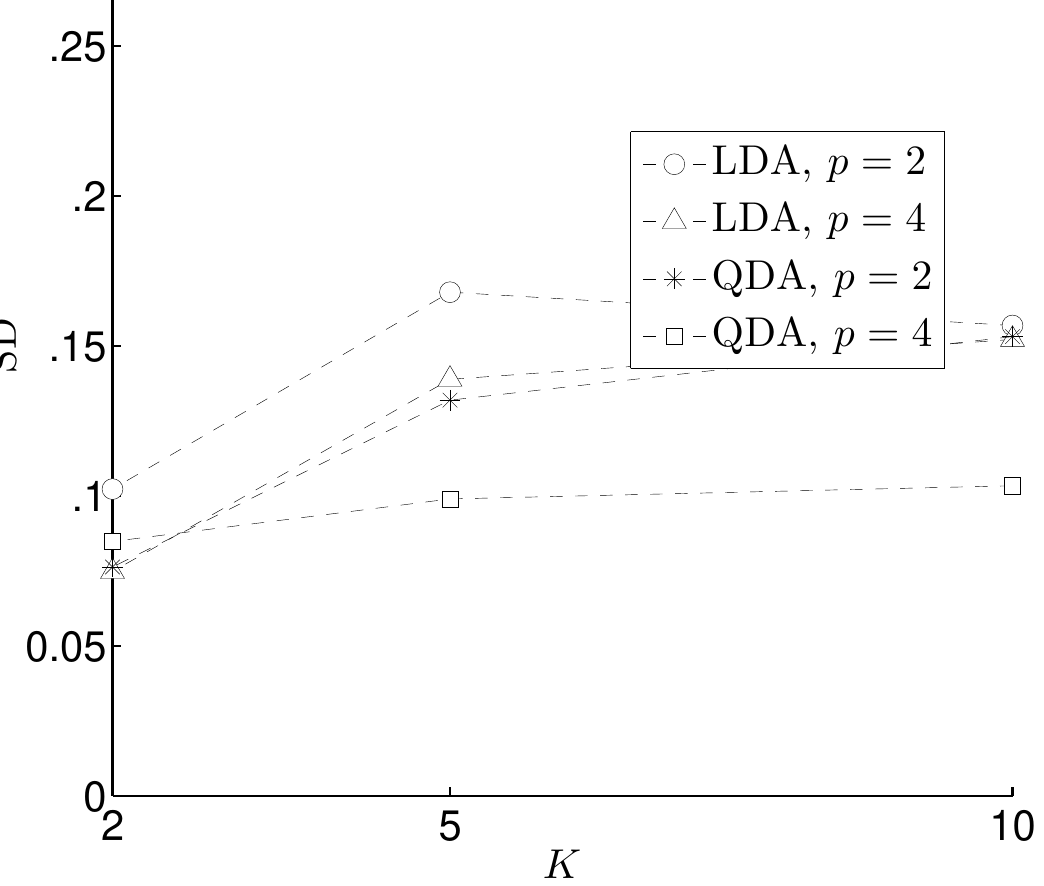}\includegraphics[width=0.16\textwidth]{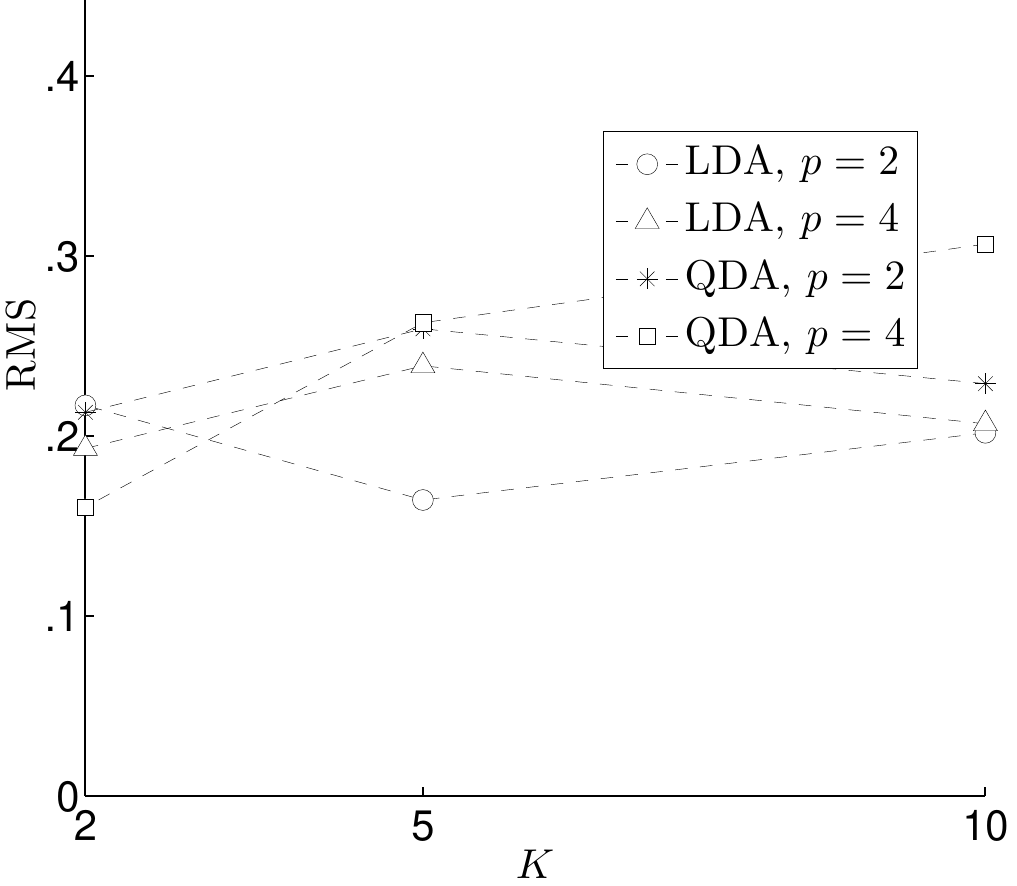}\includegraphics[width=0.16\textwidth]{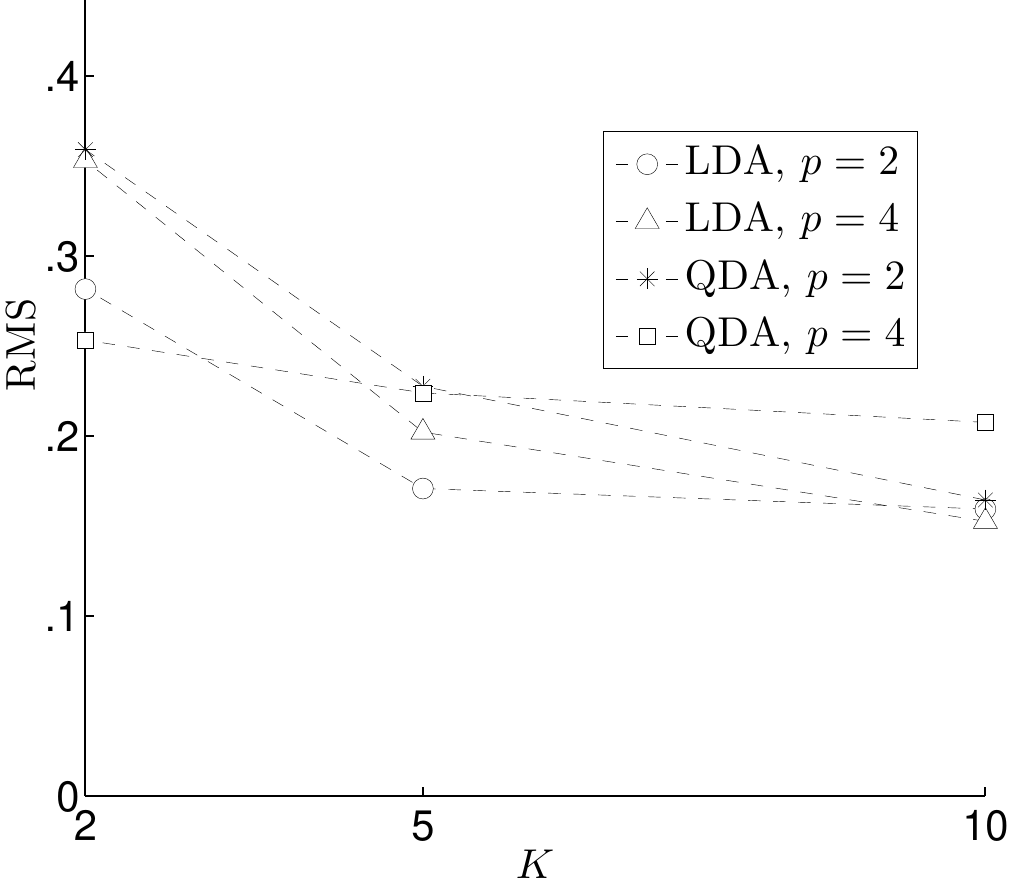}

  \includegraphics[width=0.16\textwidth]{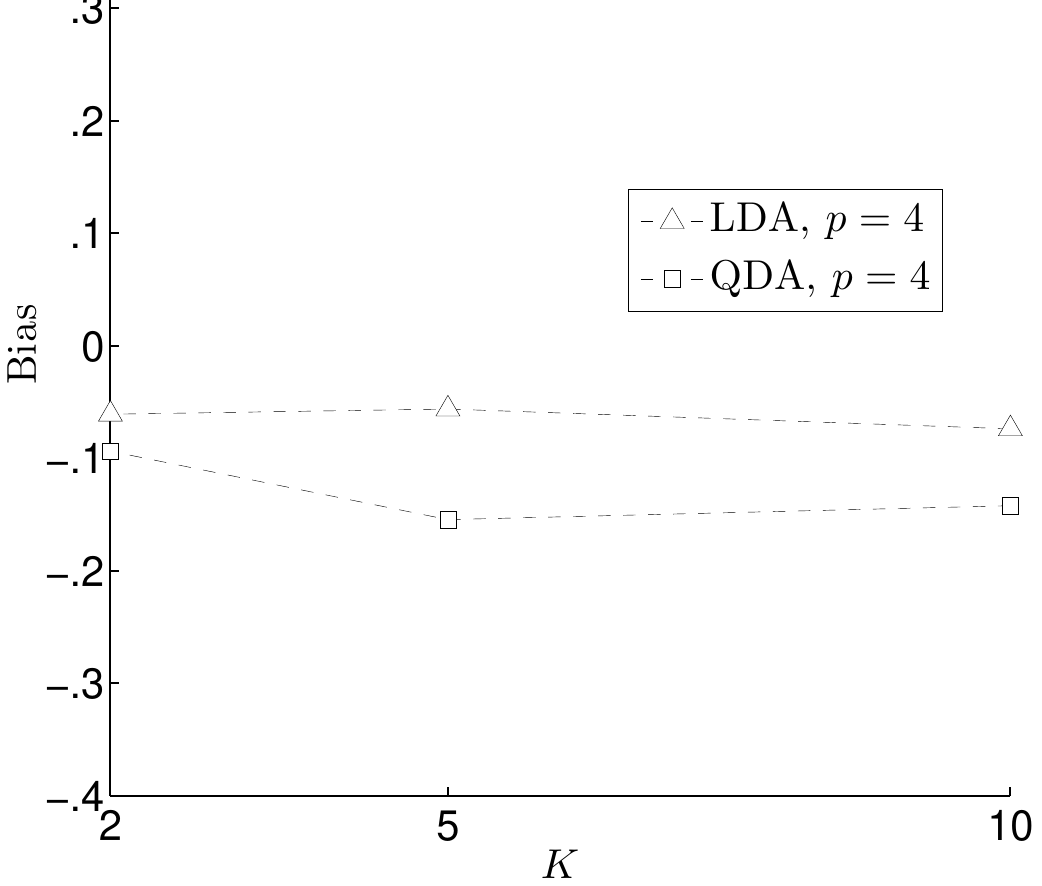}\includegraphics[width=0.16\textwidth]{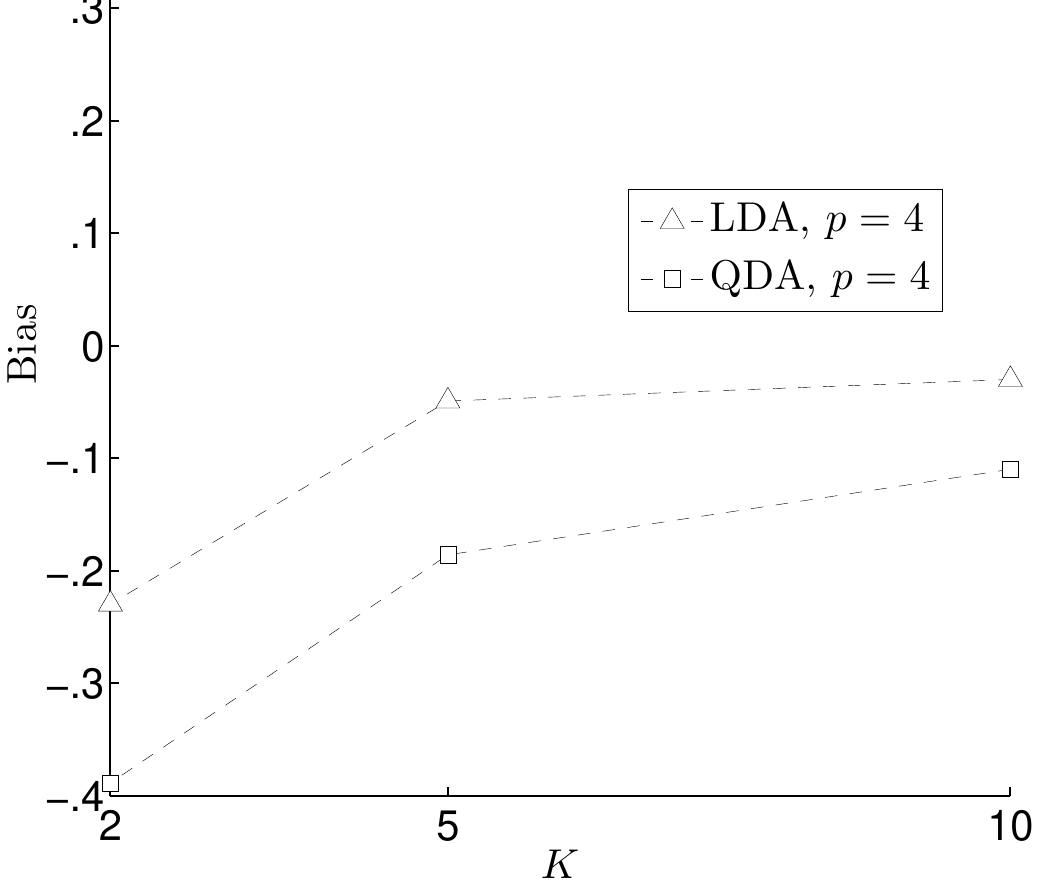}\includegraphics[width=0.16\textwidth]{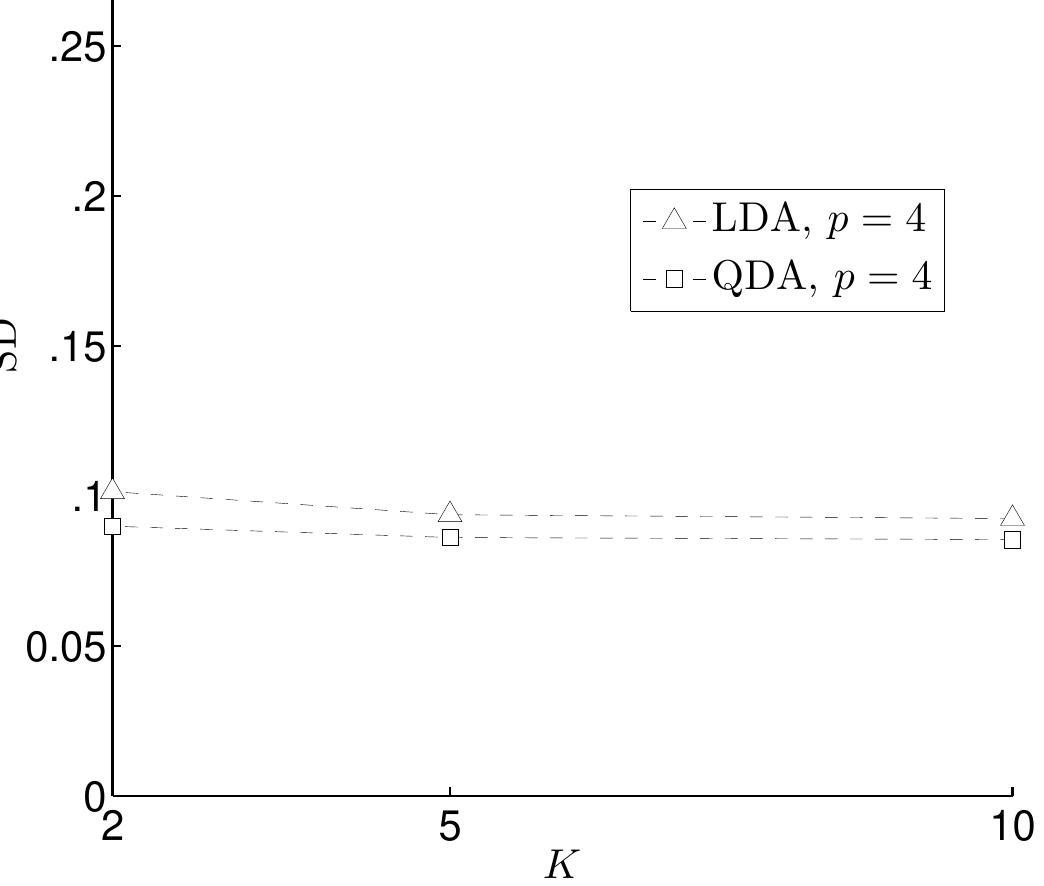}\includegraphics[width=0.16\textwidth]{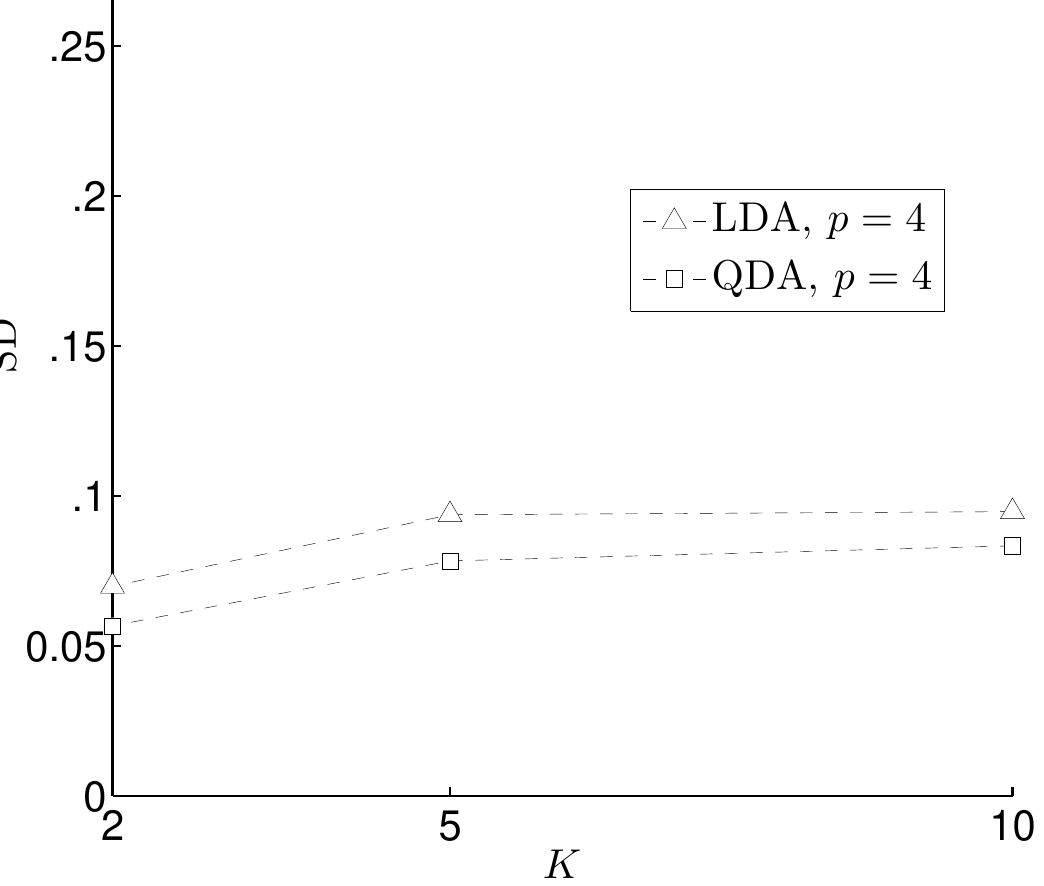}\includegraphics[width=0.16\textwidth]{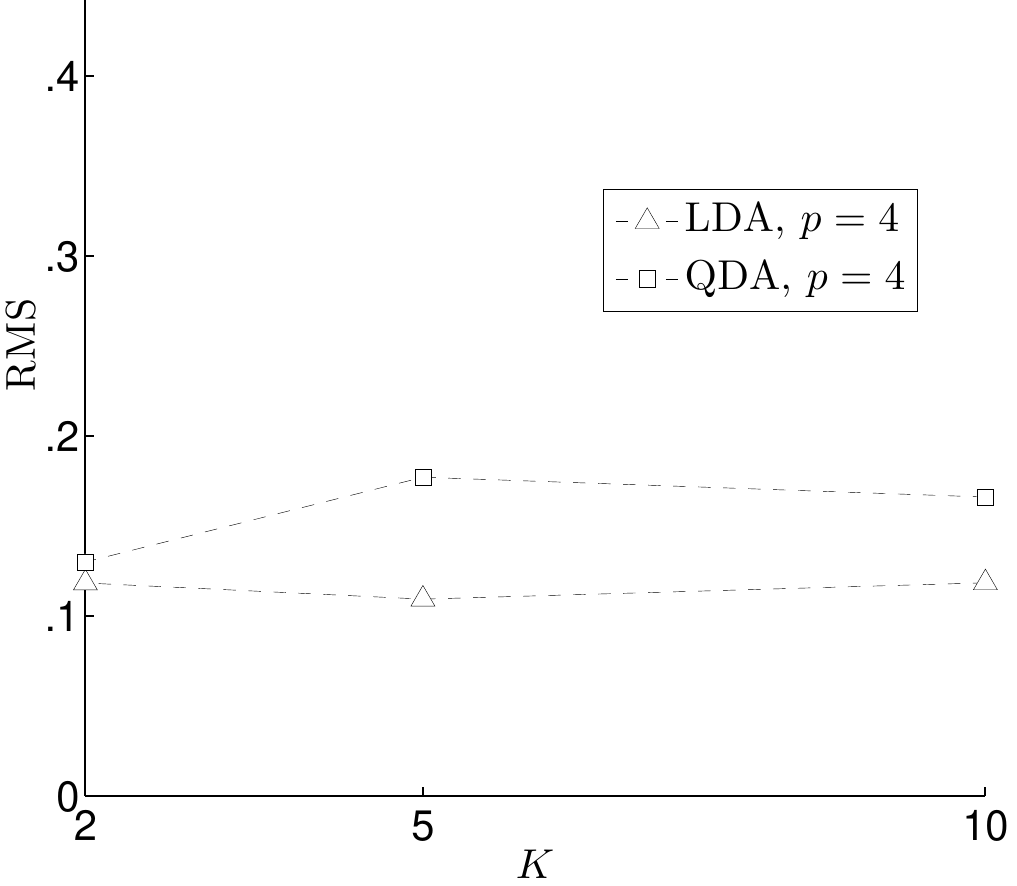}\includegraphics[width=0.16\textwidth]{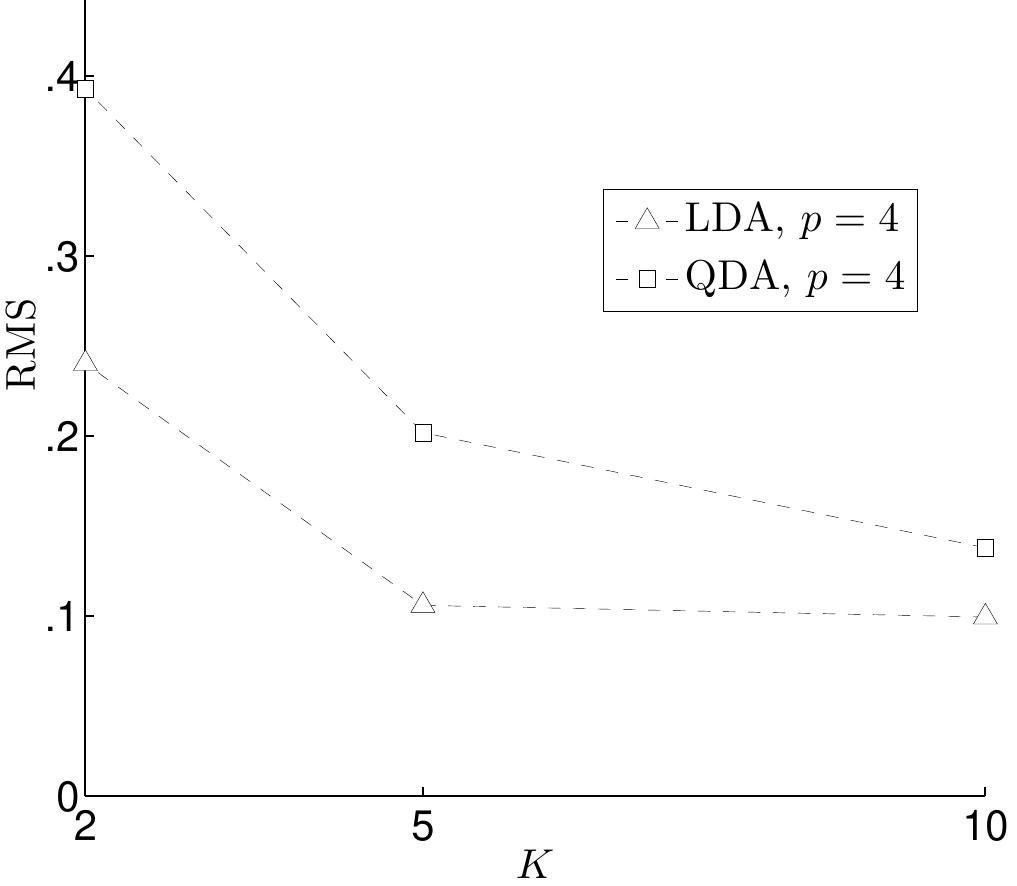}

  \caption{Comparison between the IF method and the adhoc estimator CVKR2; Bias, SD, and RMS are
    normalized to the estimand (the true SD of CVKR). The three rows of plots correspond to for
    $n=10, 20, 60$ respectively.}\label{Fign60}%
\end{figure*}

\section{Experiments and Discussion}\label{SecSimulation}
As mentioned in Section \ref{SecDisc}, we are currently planning a very comprehensive comparative
study for many versions of estimators along with their associated methods of estimating their
variance. We suffice, here, with a moderate number of experiments to explore our new IF-based method
of estimating the variance of the cross validation estimator. We tried only the LDA and QDA
classifiers with training size of 10, 20, 40, 60 per class sampled from normal distributions with 2
and 4 dimensions. We sampled observations from two normal distributions with identity covariance
matrices and zero mean vector and $c\bm{1}$ respectively, and $c$ is adjusted for class
separability.  We tried different versions of CV, namely CVKR, CVKM and CVKR with enforcing testing
on the folds $k_{1}=k_{2}$. In all experiments we repeated the CV (either in CVKR or CVKM) 1000
times. Of course CVKR converges much faster than CVKM (typically 100 repetition is adequate); we
preferred to keep both numbers of repetitions equal for the sake of comparison.

\bigskip

In general, all CV versions produce almost the same estimate of the AUC, as expected from
\cite{Yousef2019LeisurelyLookVersionsVariants-arxiv}, with the same Monte Carlo (MC) variance. For
that reason, we only show the CVKM estimate. Regarding variance estimation, the IF-based method of
the CVKM is downwardly biased with small variance. The bias decreases inversely with $K$ for the
same $n$. However, subjectively speaking, for $K=5$ we have a good compromise between the bias of
the estimate and the bias of its variance estimation. Detailed results are illustrated in
Tables~\ref{Tablen10}--\ref{Tablen60} (\APPENDIX). These tables compare the IF-based estimator with
the ad-hoc estimator (\ref{Eqadhoccvkm}). Since the variance of both the CVKM and CVKR is the same
we can use the ad-hoc estimators of CVKR \eqref{EQadhocVersions} to estimate the variance of the
CVKM as well. This comparison is illustrated in Tables \ref{Tablen10CVKR}%
--\ref{Tablen60CVKR}. From the tables, it is noticeable that the estimator
$\widehat{\operatorname*{Var}}_{2}^{CVKR}$ (that is calculated from averaging \eqref{EQadhocCVK2}
over the many repetitions of the CVKR) is better than others.

In Figures~\ref{Fign60}, we compare this estimator to the IF method using the results of tables
\ref{Tablen10CVKR}--\ref{Tablen60CVKR}. The Bias, SD, and RMS are normalized in these figure to the
estimand, i.e., the true SD of CVKR (or CVKM since they are the same for large repetitions)
calculated from the MC trials and shown in the tables.

\bigskip

From the results, it is clear that the largest bias of IF-based estimator is exercised for the small
sample size. The reason is related to the smoothness issue. E.g., for the case of $n=10$ and $K=5$
the number of training set permutations is $\binom{10}{2}=45$, which is too few! This is challenging
for the IF method, where the variance estimate is almost coming from the first term $I$
of~\eqref{EQAUCi}.

Unfortunately (this is unfortunate since we hoped that the IF method will outperform the ad hoc
estimator), the ad-hoc estimator of the 10-fold CVKM works better than the IF, which raises two
questions. First, why does the ad-hoc estimator work well? This needs more investigation and
analysis following the same route of Section \ref{SecBengio}. Second, what is the sufficient amount
of ``smoothness'' needed for the IF estimator to work well. This needs more mathematical
investigation and comparison to the LPOB estimator where the IF methods where very successful.

To get a touchy feeling of this relative smoothness issue, we compare the IF for both the bootstrap
and 10-fold CVKM. For example consider the case of QDA, $n=10,~20$, and same population parameters
as above. The IF estimator has three components, $I,~II,~$and $III$ (revisit Section
\ref{SecVarEst}). The last two terms come from the derivative of the probability of the training
sets. Large number of distinct training sets (number of permutations) is necessary for satisfactory
smoothing. Table \ref{TableBScvkm} illustrates the true SD for both estimators and the IF estimate:
once with taking into account only the first term $I$ and the other with taking into account all the
three terms $I,~II,~$and $III$. It is obvious that the first component $I$ of the IF estimator in
both cases is able to capture almost the same percentage of the variability. However, the other two
terms, $II$ and $III$, were able to capture the rest of the variability in the case of bootstrap
estimator; whereas they failed to do so in the case of CVKM. The extreme case appears when $n=10$,
where 10-fold CVKM is nothing but the CVN. In such a case the terms $II$ and $III$ are able to
capture zero percent of the variability (as was proven in Section \ref{SecVarEst}). This is a
consequence of the fact that for every observation there is only one possible training
set. (Revisiting Figure~\ref{FIGCVBSratio}) in \APPENDIX~sheds more light on the relative number of
permutations between the CVK and the luxurious bootstrap.)


\section{Conclusion and Current Work in Progress}\label{SecDisc}
We considered the problem of estimating the standard error of CV-based estimators that estimate the
performance of classification rules. We derived a novel and rigorous method, for that purpose, based
on the Influence Function (IF) approach. Although the method works well in terms of the RMS error
with some bias, the ad-hoc methods available in the literature still work better than our developed
method in the finite set of experiments that we conducted. Therefore, there is a list of interesting
tasks to complement the present article; however, each may be an article itself. We summarize them
in the following precise points:
\begin{enumerate}[partopsep=0in,parsep=0in,topsep=0.1in,itemsep=0.0in,leftmargin=0.3in]
  \item More comprehensive simulation study to figure out when the IF method does win or loose.

  \item Analyzing the different variants of the ad-hoc variance estimators in terms of the
  covariance structure of the cross validation, to figure out why they work well.

  \item More investigation of the smoothness issue to figure out the relationship between the
  ``amount of smoothness'' (measured in the different possible number of permutations) and the bias
  of the method.
\end{enumerate}
In addition, and just for mathematical interest with no practical value, it may be interesting to
derive a closed form expression for the SD of the CVKM rather than the expression that involves the
terms $I,~II,~$and $III$; and to derive the closed form expression for the probability $G$ for the
case of CVKR.



\section{Acknowledgment}\label{sec:acknoledgment}
The author is grateful to the U.S. Food and Drug Administration (FDA) for funding an earlier stage
of this project, circa 2008; to both, Brandon Gallas and Weijie Chen, of the FDA, for reviewing an
earlier stage of this manuscript when it was an internal report; and to Kyle Myers for her
continuous support.


{
  \scriptsize
  \bibliographystyle{model2-names}
  \bibliography{booksIhave,publications}
}
\clearpage

\normalfont
\section*{Appendix}

\textbf{Proof of Eq.~\eqref{EqAUCCVKMi}}

\begin{align}
  &\widehat{AUC}^{\left(  CVKM\right)  }\notag\\
  &=\frac{1}{n_{1}n_{2}}\sum\limits_{j_{2}=1}^{n_{2}}{\sum\limits_{j_{1}=1}^{n_{1}}}\left[  \left.\sum_{m}I_{j_{2}}^{m}I_{j_{1}}^{m}\psi\left(  h_{m}\left(  j_{1}\right),h_{m}\left(  j_{2}\right)  \right)  \right/  \sum_{m}I_{j_{2}}^{m}I_{j_{1}}^{m}\right] \notag\\
  &=\frac{1}{n_{1}}{\sum\limits_{j_{1}=1}^{n_{1}}}\frac{1}{n_{2}}\sum\limits_{j_{2}=1}^{n_{2}}\left[  \left.  \sum_{m}I_{j_{2}}^{m}I_{j_{1}}^{m}\psi\left(  h_{m}\left(  j_{1}\right)  ,h_{m}\left(  j_{2}\right)\right)  \right/  \sum_{m}I_{j_{2}}^{m}I_{j_{1}}^{m}\right] \notag\\
  &  =\frac{1}{n_{1}}{\sum\limits_{i=1}^{n_{1}}}\widehat{AUC}_{1i}.
\end{align}

\bigskip

\textbf{Proof of Eq.~\eqref{EQnaiveEst-Decomp}}

\begin{align}
  &\widehat{\operatorname*{Var}}\left[  \widehat{Err}^{\left(CVK\right)  }\right] =\frac{1}{K}\frac{1}{K(K-1)}\left[  (K-1)\sum_{k=1}^{K}err_{k}^{2}-\underset{k\neq k^{\prime}}{\sum\sum}err_{k}err_{k^{\prime}}\right]\notag\\
  &=\frac{1}{K}\frac{1}{K(K-1)}\left[ (K-1)\sum_{k=1}^{K}\left(\frac{1}{n_{K}}\sum_{i\in\mathcal{K}^{-1}\left( k\right) }e_{i}\right) \left(\frac{1}{n_{K}}\sum_{i^{\prime}\in\mathcal{K}^{-1}\left( k\right) }e_{i^{\prime}}\right)\right. - \notag\\
  &\left. \underset{k\neq k^{\prime}}{\sum\sum}\left( \frac{1}{n_{K}}\sum_{i\in\mathcal{K}^{-1}\left(k\right) }e_{i}\right)\left(\frac{1}{n_{K}}\sum_{i^{\prime}\in\mathcal{K}^{-1}\left( k^{\prime}\right)}e_{i^{\prime}}\right)\right]\notag\\
  &=\frac{1}{K}\frac{1}{K\left( K-1\right) }\left[ \frac{\left(K-1\right)}{n_{K}^{2}}\sum_{k=1}^{K}\left( \sum_{i\in\mathcal{K}^{-1}\left(k\right)}e_{i}^{2}+\underset{i\neq i^{\prime};i,i^{\prime}\in\mathcal{K}^{-1}\left(k\right)}{\sum\sum}e_{i}e_{i^{\prime}}\right) \right. - \notag\\
  &\left.  \frac{1}{n_{K}^{2}}\underset{k\neq k^{\prime}}{\sum\sum}\sum _{i\in\mathcal{K}^{-1}\left(k\right)}\sum_{i^{\prime}\in\mathcal{K}^{-1}\left( k^{\prime}\right) }e_{i}e_{i^{\prime}}\right].
\end{align}

\bigskip

\textbf{Proof of Eq.~\eqref{EqVarExpEst}}

\begin{align}
  &\operatorname*{E}\widehat{\operatorname*{Var}}\left[  \widehat{Err}^{\left(
    CVK\right)  }\right]  =\frac{1}{K}\frac{1}{K\left(  K-1\right)  }\times\notag\\
  &\left[
    \frac{\left(  K-1\right)  }{n_{K}^{2}}K\left(  n_{K}\left(  \sigma^{2}+\mu
    ^{2}\right)  +\left(  n_{K}^{2}-n_{K}\right)  \left(  \omega+\mu^{2}\right)
    \right)  -\left(  K^{2}-K\right)  \left(  \gamma+\mu^{2}\right)  \right]
    \nonumber\\
  &=\frac{1}{n}\sigma^{2}+\frac{\left(  n_{K}-1\right)  }{n}\omega-\frac{\gamma}{K}.
\end{align}

\begin{proof}[Proof of Lemma~\ref{LemmaBiasOfVar}]
  It is easy to decompose (\ref{EQnaiveEst}) as%
  \begin{align}
    \widehat{\operatorname*{Var}}\left[  \widehat{Err}^{\left(CVK\right)  }\right] &=\frac{1}{K}\frac{1}{K\left( K-1\right) } \times \notag\\
                                                                                  &\left[ \frac{\left(K-1\right)}{n_{K}^{2}}\sum_{k=1}^{K}\left( \sum_{i\in\mathcal{K}^{-1}\left(k\right)}e_{i}^{2}+\underset{i\neq i^{\prime};i,i^{\prime}\in\mathcal{K}^{-1}\left(k\right)}{\sum\sum}e_{i}e_{i^{\prime}}\right) \right. - \notag\\
                                                                                  &\left.  \frac{1}{n_{K}^{2}}\underset{k\neq k^{\prime}}{\sum\sum}\sum _{i\in\mathcal{K}^{-1}\left(k\right)}\sum_{i^{\prime}\in\mathcal{K}^{-1}\left( k^{\prime}\right) }e_{i}e_{i^{\prime}}\right].\label{EQnaiveEst-Decomp}
  \end{align}
  Taking the expectation of both sides and substituting from (\ref{EQcomp1})--(\ref{EQcomp3}) with%
  \begin{subequations}
    \begin{align}
      \operatorname*{E}e_{i}^{2}  &  =\sigma^{2}+\mu^{2},\\
      \operatorname*{E}\left[  e_{i}e_{i^{\prime}}\right]   &  =\omega+\mu^{2},~i\neq j,~\mathcal{K}\left(  i\right)  =\mathcal{K}\left(  j\right) \\
      \operatorname*{E}\left[  e_{i}e_{i^{\prime}}\right]   &  =\gamma+\mu^{2},~i\neq j,~\mathcal{K}\left(  i\right)  \neq\mathcal{K}\left(  j\right)
    \end{align}
  \end{subequations}%
  gives%
  \begin{align}
    \operatorname*{E}\widehat{\operatorname*{Var}}\left[  \widehat{Err}^{\left(CVK\right)  }\right]  &=\frac{1}{K}\frac{1}{K\left(  K-1\right)  }\times\notag\\
                                                                                                    &\biggl[
                                                                                                      \frac{\left(  K-1\right)  }{n_{K}^{2}}K\left(  n_{K}\left(  \sigma^{2}+\mu
                                                                                                      ^{2}\right)  +\left(  n_{K}^{2}-n_{K}\right)  \left(  \omega+\mu^{2}\right)
                                                                                                      \right)\biggr.\notag\\
                                                                                                    &\biggl.-\left(  K^{2}-K\right)  \left(  \gamma+\mu^{2}\right)  \biggr]\nonumber\\
                                                                                                    &=\frac{1}{n}\sigma^{2}+\frac{\left(  n_{K}-1\right)  }{n}\omega-\frac{\gamma}{K}.\label{EqVarExpEst}
  \end{align}
  Then, subtracting (\ref{EQvarCV}) from this previous equation gives the bias%
  \begin{align}
    &\operatorname*{Bias}\left(  \widehat{\operatorname*{Var}}\left[  \widehat{Err}^{\left(CVK\right)  }\right]  \right)=\operatorname*{E}\widehat{\operatorname*{Var}}\left[  \widehat{Err}^{\left(  CVK\right)  }\right]-\operatorname*{Var}\left[  \widehat{Err}^{\left(  CVK\right)  }\right] \notag\\
    &  =\left(  \frac{1}{n}\sigma^{2}+\frac{\left(  n_{K}-1\right)  }{n}\omega-\frac{\gamma}{K}\right)  -\left(  \frac{1}{n}\sigma^{2}+\frac{n_{K}-1}{n}\omega+\frac{n-n_{K}}{n}\gamma\right) \notag\\
    &  =-\gamma.
  \end{align}
\end{proof}

\begin{proof}[\textbf{Proof 1 of Lemma~\ref{lem:prob}}]
  We have to observe the testing fold
  and whether this fold includes the perturbed observation $i$ or not. If it appears in the testing
  fold then the first withdrawn observation has the probability $\widehat{f}_{1_{\varepsilon,i}%
  }(j_{1})=\left( 1-\varepsilon\right) /n_{1},~i\notin j_{1}$. The second will have a probability of
  $\left( 1-\varepsilon\right) /\left( n_{1}-1\right) $, and so on. Since ordering is not important,
  including the factor $(n_{1}-n_{1K})!$ takes care of all possible permutations. Therefore we have%
  \begin{align}
    g_{1\varepsilon,i}  &  =\frac{1-\varepsilon}{n_{1}}\cdot\frac{1-\varepsilon}{n_{1}-1}\cdot\cdot\cdot\frac{1-\varepsilon}{n_{1}-\left(  n_{1}-n_{1K}-1\right)  }\cdot\left(  n_{1}-n_{1K}\right)  !\\
                       &  =\frac{1}{\binom{n_{1}}{n_{1}-n_{1K}}}\left(  1-\varepsilon\right)^{\left(  n_{1}-n_{1K}\right)  },\\
    g_{1\varepsilon,i}^{\cdot}\left(  0\right)   &  =\left.  \frac{\partial g_{\varepsilon,i}}{\partial\varepsilon}\right\vert _{\varepsilon=0}\\
                       &  =\frac{-\left(  n_{1}-n_{1K}\right)  }{\binom{n_{1}}{n_{1K}}},\\
    \frac{g_{1\varepsilon,i}^{\cdot}\left(  0\right)  }{g_{10}}  &  =n_{1K}-n_{1}.\label{EQgdot1}%
  \end{align}
  On the other hand, if $i$ does not appear in the testing fold, i.e., it appears in the training set,
  this means that in addition to withdrawing $i$ we withdraw $n_{1}-n_{1K}-1$ observations. There are
  $n_{1}-n_{1K}$ permutations depending on at which of the $n_{1}-n_{1K}$ positions the observation $i$
  will appear. For each position occupied by $i$ we have $(n_{1}-n_{1K}-1)!$ different permutations of
  the other $n_{1}-n_{1K}-1$ observations. Now, we can write $g_{1\varepsilon,i}$ as a summation over
  the $n_{1}-n_{1K}$ positions as
  \begin{subequations}
    \begin{align}
      g_{_{1}\varepsilon,i}  &  =\left(  n_{1}-n_{1K}-1\right)!\times\\\notag
                           &\sum_{r=1}^{n_{1}-n_{1K}}\frac{1-\varepsilon}{n_{1}}\cdot\cdot\cdot\frac{1-\varepsilon}{n_{1}-\left(  r-2\right)  }\cdot\left[  \frac{1-\varepsilon}{n_{1}-\left(r-1\right)  }+\varepsilon\right]  \cdot\frac{1}{n_{1}-r}\cdot\cdot\cdot \notag\\
                           &\frac{1}{n_{1}-\left(  n_{1}-n_{1K}-1\right)  }\label{EQg1}\\
                           &  =\frac{\left(  n_{1}-n_{1K}-1\right)  !}{P_{\left(  n_{1}-n_{1K}\right)}^{n}}\sum_{r=1}^{n_{1}-n_{1K}}\left(  1-\varepsilon\right)  ^{\left(r-1\right)  }\left(  1-\varepsilon+\varepsilon\left(  n_{1}-\left(r-1\right)  \right)  \right)\\
                           &  =\frac{1}{\left(  n_{1}-n_{1K}\right)  \binom{n_{1}}{n_{1}-n_{1K}}}\sum_{r=1}^{n_{1}-n_{1K}}\left(  1-\varepsilon\right)  ^{\left(  r-1\right)}\left(  \varepsilon n_{1}-r\varepsilon+1\right)  .
    \end{align}
  \end{subequations}
  Notice that the first probability term appearing in (\ref{EQg1}) after the rectangular brackets is
  $1/\left( n_{1}-r\right) $ because all remaining observations now have equal probability after
  withdrawing the perturbed observation $i$. The derivative of $g_{1\varepsilon,i}$ is given by
  \begin{subequations}
    \begin{align}
      g_{1\varepsilon,i}^{\cdot}\left(  0\right)   &  =\frac{1}{\left(  n_{1}-n_{1K}\right)  \binom{n_{1}}{n_{1K}}}\sum_{r=1}^{n_{1}-n_{1K}}\left.\frac{\partial}{\partial\varepsilon}\left(  \left(  1-\varepsilon\right)^{\left(  r-1\right)  }\left(  \varepsilon n_{1}-r\varepsilon+1\right)\right)  \right\vert _{\varepsilon=0}\\
                                                 &  =\frac{1}{\left(  n_{1}-n_{1K}\right)  \binom{n_{1}}{n_{1K}}}\sum_{r=1}^{n_{1}-n_{1K}}\left(  n_{1}-2r+1\right) \\
                                                 &  =\frac{1}{\left(  n_{1}-n_{1K}\right)  \binom{n_{1}}{n_{1K}}}n_{1K}\left(n_{1}-n_{1K}\right) \\
                                                 &  =\frac{n_{1K}}{\binom{n_{1}}{n_{1K}}}. \label{EQgdot2}%
    \end{align}
  \end{subequations}
  Then, we can write
  \[\frac{g_{1\varepsilon,i}^{\cdot}\left( 0\right) }{g_{10}}=n_{1K}.\]%
  Compactly, we can compile (\ref{EQgdot1}) and (\ref{EQgdot2}) as%
  \begin{subequations}
    \begin{align}
      \frac{g_{1\varepsilon,i}^{\cdot}\left(  0\right)  }{g_{10}}  &  =\left(n_{1K}-n_{1}\right)  I_{i}^{m}+n_{1K}\left(  1-I_{i}^{m}\right) \\
                                                                &  =n_{1K}-I_{i}^{m}n_{1}, \label{EQgdot}%
    \end{align}
  \end{subequations}
  where $I_{i}^{m}$ is an indicator for whether $i$ belongs to the first testing fold of the repetition
  $m$.
\end{proof}

\begin{proof}[\textbf{Proof 2 of Lemma~\ref{lem:prob}}] This other approach relies on a different
  perturbation method, but gives very similar results to the expressions of the Lemma; however, it
  lacks a closed form expression.

  The probability $g_{\varepsilon,i}$ of a training set, i.e., the
  $K-1$ folds is the same as the probability of obtaining one testing partition; hence for $i=j_{1}$
  we have
  $g_{\varepsilon,i} =\frac{1}{\binom{n_{1}-1}{n_{1K}-1}}\cdot\frac{1}{\binom{n_{2}}{n_{2K}}}$ and
  $g_{\varepsilon,i}^{\cdot} =0$. If $i\neq j_{1}$ then either $i$ is sampled in the same partition
  of $j_{1}$ (i.e., not included in the training set) or not (i.e., included in the training
  set). First, suppose that $i$ is included in the training set. We now sample $n_{1K}-1$ from
  $n_{1}-1$ observations. Not all of the $n_{1K}-1$ have equal probabilities; each observation
  $i^{\prime}$ of them has a probability
  $f_{\varepsilon,i}\left( i^{\prime}\right)
  =\frac{\frac{1-\varepsilon}{n_{1}}+\varepsilon\delta_{ii^{\prime}}}{1-\frac{1-\varepsilon}{n_{1}}}
  =\allowbreak\frac{1+\varepsilon n_{1}\delta_{ii^{\prime}}-\varepsilon}{\varepsilon+n_{1}-1}$,
  where the normalization factor $1-\frac{1-\varepsilon}{n_{1}}$ accounts for the left-out
  observation $j_{1}$. We will sample the $n_{1K}-1$ observations successively and normalize by the
  remaining probability measure after each withdraw. For the following discussion, let
  $ p =f_{\varepsilon,i}\left( i^{\prime}\neq i\right)
  =\frac{1-\varepsilon}{\varepsilon+n_{1}-1}$. Then,
  $ p_{i} =f_{\varepsilon,i}\left( i^{\prime}=i\right) =\allowbreak\frac{1+\varepsilon
    n_{1}-\varepsilon}{\varepsilon+n_{1}-1}$. Then, $g_{\varepsilon,i}$ and its derivative are given
  by%
  \begin{align}
    g_{\varepsilon,i}  &  =p\cdot\frac{p}{1-p}\cdot\frac{p}{1-2p}\cdot\cdot\cdot\frac{p}{1-\left(  n_{1K}-2\right)  p}\cdot\left(  n_{1K}-1\right)!\cdot\frac{1}{\binom{n_{2}}{n_{2K}}}\notag\\
                      &  =\frac{p^{\left(  n_{1K}-1\right)  }}{\prod\limits_{r=1}^{n_{1K}-2}1-rp}\cdot\left(  n_{1K}-1\right)  !\cdot\frac{1}{\binom{n_{2}}{n_{2K}}},\notag\\
    \left.  g_{\varepsilon,i}^{\cdot}\right\vert _{\varepsilon=0}  &  =\left.\frac{\partial g_{\varepsilon,i}}{\partial p}\frac{\partial p}{\partial \varepsilon}\right\vert _{\varepsilon=0}\notag\\
                      &=\left.  \frac{\partial g_{\varepsilon,i}}{\partial p}\right\vert_{p=\frac{1}{n_{1}-1}}\cdot\frac{-n_{1}}{\left(  n_{1}-1\right)  ^{2}}\cdot\left(  n_{1K}-1\right)  !\cdot\frac{1}{\binom{n_{2}}{n_{2K}}}\notag\\
                      &  =\frac{1}{\binom{n_{2}}{n_{2K}}}A.\notag
  \end{align}
  It should be obvious that the value
  $g_{\varepsilon,i}\left( 0\right)
  =\frac{1}{\binom{n_{2}}{n_{2K}}}\cdot\frac{1}{\binom{n_{1}-1}{n_{1K}-1}}$. When the
  $i$\textsuperscript{th} observation included in the testing fold this means that in addition to this
  $i$\textsuperscript{th} observation, we withdraw $n_{1K}-2$ observations. There are $n_{1K}-1$
  permutations depending on in which $n_{1K}-1$ position the observation $i$ will appear. If it
  appears in the $r^{\text{th}}$ position then $g_{\varepsilon,i}$ will be equal to the summation over
  these $n_{1K}-1$ permutations giving%
  \begin{align}
    g_{\varepsilon,i}&=\frac{1}{\binom{n_{2}}{n_{2K}}}\cdot P_{_{\left(n_{1K}-2\right)  }}^{\left(  n_{1}-2\right)  }\cdot \sum_{r=1}^{n_{1K}-1}\frac{p}{1}\cdot\frac{p}{1-p}\cdot\cdot\cdot\frac{p}{1-\left(  r-2\right)  p}\cdot\notag\\
                    &\frac{p_{i}}{1-\left(  r-1\right)  p}\cdot\frac{p}{1-\left(  \left(  r-1\right)  p+p_{i}\right)  }\cdot\cdot\cdot\frac{p}{1-\left(  \left(  n_{1K}-3\right)  p+p_{i}\right)  }\notag\\
                    &  =\frac{p^{\left(  n_{1K}-2\right)  }p_{i}\binom{n_{1}-2}{_{n_{1K}-2}}\left(  n_{1K}-2\right)  !}{\binom{n_{2}}{n_{2K}}}\times\notag\\
                    &\sum_{r=1}^{n_{1K}-1}\frac{1}{\prod\limits_{j=1}^{r}\left(  1-(j-1)p\right)  \prod\limits_{j=r+1}^{n_{1K}-1}\left(  1-\left(  j-2\right)  p-p_{i}\right)  },\notag
  \end{align}
  which leads to $g_{\varepsilon,i}^{\cdot} =\frac{1}{\binom{n_{2}}{n_{2K}}}B$. The factorial terms
  account for the permutations of the observations before and after the appearance of $i$. Combining
  the above equations leads to
  \begin{equation}
    g_{\varepsilon,i}^{\cdot}=\frac{1}{\binom{n_{2}}{n_{2K}}}\left(1-\delta_{ij_{1}}\right)  \left(  A\left(  1-I_{i}\right)  +BI_{i}\right),
  \end{equation}
  where $I_{i}$ indicates whether $i$ is included in the testing set or not.
\end{proof}

\textbf{Proof of Eq.~\eqref{eq:9}}

\begin{align}
  \widehat{AUC}^{\left(  CVKR\right)  }  &  =\frac{1}{n_{1}n_{2}}\sum_{j_{2}=1}^{n_{2}}\sum_{j_{1}=1}^{n_{1}}\left[  \left.  \sum_{m}\psi\left(h_{m}\left(  j_{1}\right)  ,h_{m}\left(  j_{2}\right)  \right)  \right/M\right]  ,\notag\\
  \widehat{AUC}_{\varepsilon,i}^{\left(  CVKR\right)  }  &  =\sum\limits_{j_{2}=1}^{n_{2}}{\sum\limits_{j_{1}=1}^{n_{1}}{\hat{f}_{1_{\varepsilon,i}}(j_{1})\hat{f}_{2_{\varepsilon,i}}(j_{2})}}\left[  \sum_{m}\psi\left(h_{m}\left(  j_{1}\right)  ,h_{m}\left(  j_{2}\right)  \right)  G_{\varepsilon,i}\right]  ,\notag\\
                                        &  =\sum\limits_{j_{2}=1}^{n_{2}}{\sum\limits_{j_{1}=1}^{n_{1}}{A}}\left(\varepsilon\right)  B\left(  \varepsilon\right).\notag
\end{align}
Then.
\begin{align}
  \hat{U}_{1_{i}}  &  =\sum_{j_{2}}\sum_{j_{1}}\left(  A^{\cdot}\left(0\right)  B\left(  0\right)  +A\left(  0\right)  B^{\cdot}\left(  0\right)\right)\notag\\
                 &  =\sum\limits_{j_{2}=1}^{n_{2}}{\sum\limits_{j_{1}=1}^{n_{1}}}\left[\frac{1}{n_{2}}\left(  \delta_{ij_{1}}-1/n_{1}\right)  \left[  \sum_{m}\psi\left(  h_{m}\left(  j_{1}\right)  ,h_{m}\left(  j_{2}\right)  \right)G_{0}\right]+\right.\notag\\
                 &\left. \frac{1}{n_{1}n_{2}}\sum_{m}\psi\left(  h_{m}\left(j_{1}\right)  ,h_{m}\left(  j_{2}\right)  \right)  \left.  G_{\varepsilon,i}^{\cdot}\right\vert _{\varepsilon=0}\right] \\
                 &  =\widehat{AUC}_{1i}-\widehat{AUC}^{\left(  CVKR\right)  }+ \notag\\
                 &\frac{1}{n_{1}n_{2}}\sum\limits_{j_{2}=1}^{n_{2}}{\sum\limits_{j_{1}=1}^{n_{1}}}\sum_{m}\psi\left(  h_{m}\left(  j_{1}\right)  ,h_{m}\left(  j_{2}\right)\right)  \left.  G_{\varepsilon,i}^{\cdot}\right\vert _{\varepsilon=0},
\end{align}

\begin{proof}[Proof of Lemma~\ref{LemmaSD}]
  $\hat{U}_{1_{i}}=\partial AUC_{\varepsilon,i}/\delta\varepsilon$ is given by
  \begin{align}
    \hat{U}_{1_{i}}&=\underset{I}{\underbrace{\sum\limits_{j_{2}=1}^{n_{2}}
                   {\sum\limits_{j_{1}=1}^{n_{1}}}A^{\cdot}\left( 0\right) \frac{B\left( 0\right) }{C\left( 0\right)
                   }}}+\underset{II}{\underbrace{\sum
                   \limits_{j_{2}=1}^{n_{2}}{\sum\limits_{j_{1}=1}^{n_{1}}A}\left( 0\right) \frac{B^{\cdot}\left(
                   0\right) }{C\left( 0\right) }}}-\notag\\
                 &\underset{III}{\underbrace{\sum\limits_{j_{2}=1}^{n_{2}}{\sum\limits_{j_{1}=1}^{n_{1}}A}\left( 0\right) \frac{B\left( 0\right) C^{\cdot}\left( 0\right) }{C^{2}\left( 0\right)}}},~i=1,\ldots,n_{1},
  \end{align}
  \begin{align}
    I  &  =\sum\limits_{j_{2}=1}^{n_{2}}{\sum\limits_{j_{1}=1}^{n_{1}}}\frac{1}{n_{2}}\left(  \delta_{ij_{1}}-1/n_{1}\right) \times\notag\\
       &\left[  \left.  \sum_{m}I_{j_{2}}^{m}I_{j_{1}}^{m}\psi\left(  h_{m}\left(  j_{1}\right),h_{m}\left(  j_{2}\right)  \right)  \right/  \sum_{m}I_{j_{2}}^{m}I_{j_{1}}^{m}\right] \notag\\
       &  =\widehat{AUC}_{1i}-\widehat{AUC}^{\left(  CVKM\right)  },\\
    \widehat{AUC}_{1i}  &  =\frac{1}{n_{2}}\sum\limits_{j_{2}=1}^{n_{2}}\left[\left.  \sum_{m}I_{j_{2}}^{m}I_{i}^{m}\psi\left(  h_{m}\left(  i\right),h_{m}\left(  j_{2}\right)  \right)  \right/  \sum_{m}I_{j_{2}}^{m}I_{i}^{m}\right] ,\notag
  \end{align}
  \begin{align}
    II  &  =\frac{1}{n_{1}n_{2}}\sum\limits_{j_{2}=1}^{n_{2}}{\sum\limits_{j_{1}=1}^{n_{1}}}\left[  \frac{\sum_{m}I_{j_{2}}^{m}I_{j_{1}}^{m}\psi\left(h_{m}\left(  j_{1}\right)  ,h_{m}\left(  j_{2}\right)  \right)  g_{\varepsilon,i}^{\cdot}\left(  0\right)  }{\sum_{m}I_{j_{2}}^{m}I_{j_{1}}^{m}g\left(0\right)  }\right] \notag\\
        &  =\frac{1}{n_{1}n_{2}}\sum\limits_{j_{2}=1}^{n_{2}}{\sum\limits_{j_{1}=1}^{n_{1}}}\left[  \frac{\sum_{m}I_{j_{2}}^{m}I_{j_{1}}^{m}\psi\left(h_{m}\left(  j_{1}\right)  ,h_{m}\left(  j_{2}\right)  \right)  \frac{g_{\varepsilon,i}^{\cdot}\left(  0\right)  }{g\left(  0\right)  }}{\sum_{m}I_{j_{2}}^{m}I_{j_{1}}^{m}}\right] \notag
  \end{align}
  \begin{align}
    III  &  =\frac{1}{n_{1}n_{2}}\sum\limits_{j_{2}=1}^{n_{2}}{\sum\limits_{j_{1}=1}^{n_{1}}}\frac{\left(  \sum_{m}I_{j_{2}}^{m}I_{j_{1}}^{m}\psi\left(h_{m}\left(  j_{1}\right)  ,h_{m}\left(  j_{2}\right)  \right)  g\left(0\right)  \right)  \left(  \sum_{m}I_{j_{2}}^{m}I_{j_{1}}^{m}g_{\varepsilon,i}^{\cdot}\left(  0\right)  \right)  }{\left(  \sum_{m}I_{j_{2}}^{m}I_{j_{1}}^{m}g\left(  0\right)  \right)  ^{2}}\notag\\
         &  =\frac{1}{n_{1}n_{2}}\sum\limits_{j_{2}=1}^{n_{2}}{\sum\limits_{j_{1}=1}^{n_{1}}}\frac{\left(  \sum_{m}I_{j_{2}}^{m}I_{j_{1}}^{m}\psi\left(h_{m}\left(  j_{1}\right)  ,h_{m}\left(  j_{2}\right)  \right)  \right)\left(  \sum_{m}I_{j_{2}}^{m}I_{j_{1}}^{m}\frac{g_{\varepsilon,i}^{\cdot}\left(  0\right)  }{g\left(  0\right)  }\right)  }{\left(  \sum_{m}I_{j_{2}}^{m}I_{j_{1}}^{m}\right)  ^{2}}.\notag
  \end{align}
  Similar expression is immediate for $\hat{U}_{2_{j}},~j=1,\ldots n_{2}$; and finally
  \begin{equation}
    \widehat{SD}\left[  \widehat{AUC}^{\left(  CVKM\right)  }\right]  =\sqrt
    {\frac{1}{n_{1}^{2}}\sum\limits_{i=1}^{n_{1}}{\hat{U}_{1_{i}}^{2}}+\frac
      {1}{n_{2}^{2}}\sum\limits_{j=1}^{n_{2}}{\hat{U}_{2_{j}}^{2}}}.
  \end{equation}

\end{proof}

\begin{figure}[tbh]\centering
  \includegraphics[height=1.5in]{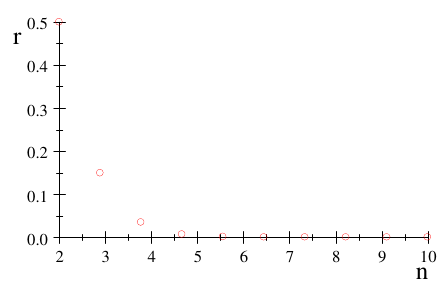}
  \caption{The ratio $r=\binom{n}{n/2}/n^{n}$ is the ratio between the number
    of permutations of both KCV (where $K=2$) and the Bootstrap.}\label{FIGCVBSratio}
\end{figure}

\clearpage

\begin{table*}[tbh] \centering
  \caption{$n_1=n_2=10$}\label{Tablen10}%
\begin{tabular}
    [c]{ccccccccccccc}%
    Classifier & LDA & LDA & LDA & LDA & LDA & LDA & QDA & QDA & QDA & QDA & QDA &
                                                                                   QDA\\
    \multicolumn{1}{l}{$p$} & \multicolumn{1}{r}{2} & \multicolumn{1}{r}{2} &
                                                                              \multicolumn{1}{r}{2} & \multicolumn{1}{r}{4} & \multicolumn{1}{r}{4} &
                                                                                                                                                      \multicolumn{1}{r}{4} & \multicolumn{1}{r}{2} & \multicolumn{1}{r}{2} &
                                                                                                                                                                                                                              \multicolumn{1}{r}{2} & \multicolumn{1}{r}{4} & \multicolumn{1}{r}{4} &
                                                                                                                                                                                                                                                                                                      \multicolumn{1}{r}{4}\\
\end{tabular}
\end{table*}%

\begin{table*}[tbh] \centering
  \begin{tabular}
    [c]{ccccccccccccc}%
    Classifier & LDA & LDA & LDA & LDA & LDA & LDA & QDA & QDA & QDA & QDA & QDA &
                                                                                   QDA\\
    \multicolumn{1}{l}{$p$} & \multicolumn{1}{r}{2} & \multicolumn{1}{r}{2} &
                                                                              \multicolumn{1}{r}{2} & \multicolumn{1}{r}{4} & \multicolumn{1}{r}{4} &
                                                                                                                                                      \multicolumn{1}{r}{4} & \multicolumn{1}{r}{2} & \multicolumn{1}{r}{2} &
                                                                                                                                                                                                                              \multicolumn{1}{r}{2} & \multicolumn{1}{r}{4} & \multicolumn{1}{r}{4} &
                                                                                                                                                                                                                                                                                                      \multicolumn{1}{r}{4}\\
\end{tabular}
  \caption{$n_1=n_2=20$}\label{Tablen20}%
\end{table*}%

\begin{table*}[tbh] \centering
  \caption{$n_1=n_2=60$}\label{Tablen60}%
  \begin{tabular}
    [c]{ccccccc}%
    Classifier & LDA & LDA & LDA & QDA & QDA & QDA\\
    \multicolumn{1}{l}{$p$} & \multicolumn{1}{r}{4} & \multicolumn{1}{r}{4} &
                                                                              \multicolumn{1}{r}{4} & \multicolumn{1}{r}{4} & \multicolumn{1}{r}{4} &
                                                                                                                                                      \multicolumn{1}{r}{4}\\
    \multicolumn{1}{l}{$K$} & \multicolumn{1}{r}{10} & \multicolumn{1}{r}{5} &
                                                                               \multicolumn{1}{r}{2} & \multicolumn{1}{r}{10} & \multicolumn{1}{r}{5} &
                                                                                                                                                        \multicolumn{1}{r}{2}\\
    \multicolumn{1}{l}{$AUC$} & \multicolumn{1}{r}{.7956} &
                                                            \multicolumn{1}{r}{.7953} & \multicolumn{1}{r}{.7954} &
                                                                                                                    \multicolumn{1}{r}{.7741} & \multicolumn{1}{r}{.7743} &
                                                                                                                                                                            \multicolumn{1}{r}{.7746}\\
\end{tabular}
\end{table*}%

\begin{table*}[tbh] \centering
  \caption{Comparison between the bootstrap (BS) and CVKM IF-based estimators in terms of the three
    components $I,\ II$ and $III$.}\label{TableBScvkm}%
  \begin{tabular}
    [c]{ccccc}%
    $n$ & Est. & $\widehat{SD}_{IF}\left(  I,II,III\right)  $ & $\widehat{SD}%
                                                               _{IF}\left(  I\right)  $ & $\operatorname*{SD}$\\
    \multicolumn{1}{r}{10} & \multicolumn{1}{l}{BS} & \multicolumn{1}{r}{.1327} &
                                                                                  \multicolumn{1}{r}{.0806} & \multicolumn{1}{r}{.1362}\\
    \multicolumn{1}{r}{10} & \multicolumn{1}{l}{CVKM} & \multicolumn{1}{r}{.1016}
                                                             & \multicolumn{1}{r}{.1016} & \multicolumn{1}{r}{.1472}\\
    \multicolumn{1}{r}{20} & \multicolumn{1}{l}{BS} & \multicolumn{1}{r}{.0908} &
                                                                                  \multicolumn{1}{r}{.0628} & \multicolumn{1}{r}{.0929}\\
    \multicolumn{1}{r}{20} & \multicolumn{1}{l}{CVKM} & \multicolumn{1}{r}{.0715}
                                                             & \multicolumn{1}{r}{.0705} & \multicolumn{1}{r}{.0869}%
  \end{tabular}
\end{table*}%

\begin{table*}[tbh] \centering
  \caption{$n_1=n_2=10$}\label{Tablen10CVKR}%
  \begin{tabular}
    [c]{ccccccccccccc}%
    Classifier & LDA & LDA & LDA & LDA & LDA & LDA & QDA & QDA & QDA & QDA & QDA &
                                                                                   QDA\\
    \multicolumn{1}{l}{$p$} & \multicolumn{1}{r}{2} & \multicolumn{1}{r}{2} &
                                                                              \multicolumn{1}{r}{2} & \multicolumn{1}{r}{4} & \multicolumn{1}{r}{4} &
                                                                                                                                                      \multicolumn{1}{r}{4} & \multicolumn{1}{r}{2} & \multicolumn{1}{r}{2} &
                                                                                                                                                                                                                              \multicolumn{1}{r}{2} & \multicolumn{1}{r}{4} & \multicolumn{1}{r}{4} &
                                                                                                                                                                                                                                                                                                      \multicolumn{1}{r}{4}\\
\end{tabular}
\end{table*}%

\begin{table*}[tbh] \centering
  \caption{$n_1=n_2=20$}\label{Tablen20CVKR}%
  \begin{tabular}
    [c]{ccccccccccccc}%
    Classifier & LDA & LDA & LDA & LDA & LDA & LDA & QDA & QDA & QDA & QDA & QDA &
                                                                                   QDA\\
    \multicolumn{1}{l}{$p$} & \multicolumn{1}{r}{2} & \multicolumn{1}{r}{2} &
                                                                              \multicolumn{1}{r}{2} & \multicolumn{1}{r}{4} & \multicolumn{1}{r}{4} &
                                                                                                                                                      \multicolumn{1}{r}{4} & \multicolumn{1}{r}{2} & \multicolumn{1}{r}{2} &
                                                                                                                                                                                                                              \multicolumn{1}{r}{2} & \multicolumn{1}{r}{4} & \multicolumn{1}{r}{4} &
                                                                                                                                                                                                                                                                                                      \multicolumn{1}{r}{4}\\
\end{tabular}
\end{table*}%

\begin{table*}[tbh] \centering
  \caption{$n_1=n_2=60$}\label{Tablen60CVKR}%
  \begin{tabular}
    [c]{ccccccc}%
    Classifier & LDA & LDA & LDA & QDA & QDA & QDA\\
    \multicolumn{1}{l}{$p$} & \multicolumn{1}{r}{4} & \multicolumn{1}{r}{4} &
                                                                              \multicolumn{1}{r}{4} & \multicolumn{1}{r}{4} & \multicolumn{1}{r}{4} &
                                                                                                                                                      \multicolumn{1}{r}{4}\\
    \multicolumn{1}{l}{$K$} & \multicolumn{1}{r}{10} & \multicolumn{1}{r}{5} &
                                                                               \multicolumn{1}{r}{2} & \multicolumn{1}{r}{10} & \multicolumn{1}{r}{5} &
                                                                                                                                                        \multicolumn{1}{r}{2}\\
    \multicolumn{1}{l}{$AUC$} & \multicolumn{1}{r}{.7956} &
                                                            \multicolumn{1}{r}{.7953} & \multicolumn{1}{r}{.7954} &
                                                                                                                    \multicolumn{1}{r}{.7741} & \multicolumn{1}{r}{.7743} &
                                                                                                                                                                            \multicolumn{1}{r}{.7746}\\
\end{tabular}
\end{table*}%



\end{document}